\documentclass{article}

\usepackage[english]{babel}
\usepackage[letterpaper,top=2cm,bottom=2cm,left=3cm,right=3cm,marginparwidth=1.75cm]{geometry}

\usepackage[round]{natbib}
\usepackage{amsmath}
\usepackage{physics}
\usepackage{mathtools}
\usepackage{enumerate}
\usepackage{graphicx}
\usepackage[colorlinks=true, allcolors=blue]{hyperref}
\usepackage{amssymb, amsthm}
\usepackage{authblk}

\newtheorem{lemma}{Lemma}
\newtheorem{definition}[lemma]{Definition}
\newtheorem{remark}[lemma]{Remark}
\newtheorem{proposition}[lemma]{Proposition}
\newtheorem{theorem}[lemma]{Theorem}

\newtheorem{assumption}{Assumption}

\theoremstyle{definition}
\newtheorem{example}{Example}

\title{Non-asymptotic error bounds for probability flow ODEs under weak log-concavity}
\author{Gitte Kremling\thanks{Email: \texttt{gitte.kremling@uni-hamburg.de}}}
\author{Francesco Iafrate\thanks{Email: \texttt{francesco.iafrate@uni-hamburg.de}}}
\author{Mahsa Taheri\thanks{Email: \texttt{mahsa.taheri@uni-hamburg.de}}}
\author{Johannes Lederer\thanks{Email: \texttt{johannes.lederer@uni-hamburg.de}}}
\affil{Department of Mathematics, University of Hamburg, Germany}
\date{}

\newcommand{\dotprod}[2]{\left\langle #1, #2 \right\rangle}
\newcommand{\R}{\mathbb{R}}
\newcommand{\E}[1]{\mathbb{E}\qty[#1]}
\newcommand{\Var}[1]{\text{Var}\qty(#1)}
\newcommand{\half}{\frac{1}{2}}
\newcommand{\de}{\, \mathrm d }
\newcommand{\intk}{\int_{t_{k-1}}^{t_k}}
\newcommand{\intj}{\int_{t_{j-1}}^{t_j}}
\newcommand{\intkT}{\int_{T-t_k}^{T-t_{k-1}}}
\newcommand{\law}[1]{\mathcal{L}\qty(#1)}
\newcommand{\bigo}[1]{\mathcal{O}\qty(#1)}
\newcommand{\eps}{\varepsilon}
\newcommand\numberthis{\addtocounter{equation}{1}\tag{\theequation}}

\begin{document}
\maketitle
\begin{abstract}
Score-based generative modeling, implemented through probability flow ODEs, has shown impressive results in numerous practical settings.
However, most convergence guarantees rely on restrictive regularity assumptions on the target distribution—such as strong log-concavity or bounded support.  This work establishes non-asymptotic convergence bounds in the 2-Wasserstein distance for a general class of probability flow ODEs under considerably weaker assumptions: weak log-concavity and Lipschitz continuity of the score function. Our framework accommodates non-log-concave distributions, such as Gaussian mixtures, and explicitly accounts for initialization errors, score approximation errors, and effects of discretization via an exponential integrator scheme. %
Bridging a key theoretical challenge in diffusion-based generative modeling, our results extend convergence theory to more realistic data distributions and practical ODE solvers.
We provide concrete guarantees for the efficiency and correctness of the sampling algorithm, complementing the empirical success of diffusion models with rigorous theory. Moreover, from a practical perspective, our explicit rates might be helpful in choosing hyperparameters, such as the step size in the discretization.
\end{abstract}

\section{Introduction}
Diffusion models are a powerful class of generative models designed to sample from complex data distributions. They operate by reversing a forward stochastic process that progressively transforms data into noise. The generative process is typically modeled using a reverse-time stochastic differential equation (SDE) or an equivalent deterministic probability flow ordinary differential equation (ODE) that preserves the same marginal distributions~\citep{song2019generative,song2021score,ho2020denoising}. The key idea is to use a learned score function—an estimate of the gradient (with respect to the data) of the log-density—to guide the reverse dynamics. Samples are then generated by integrating this reverse process from pure noise back to the data manifold. 

The key issue in diffusion models is: under what assumptions and in which settings do these reverse processes converge to the target distribution?
While a growing body of literature addresses this issue, often distinguishing between stochastic and deterministic samplers, most analyses rely on strict assumptions about the unknown target distribution—such as log-concavity or bounded support~\citep{block2020generative,de2022convergence,lee2023convergence,gao2025wasserstein}. A natural and intriguing question is whether—and how—these assumptions can be relaxed. 
In this paper, we provide an answer to this question %
for probability flow ODEs, establishing a convergence result that merely requires \textit{weak log-concavity} of the data distribution. %
This generalization allows, for example, for multi-modality---which is often expected in practice.

\paragraph{Contributions}
We study the %
distance between the approximated and the true sample distribution for a general class of probability flow ODEs, while relaxing the standard strong log-concavity assumption.
Additionally, we account for the discretization error by employing an exponential integrator discretization approach.
Our main contributions are:
\begin{enumerate}
    \item We establish 2-Wasserstein convergence bounds for a general class of probability flow ODEs under a \emph{weak concavity} and a \emph{Lipschitz} condition on the score function (Theorem~\ref{thm:main_thm}). Our results cover a broad range of data distributions, including mixtures of Gaussians. Notably, we show that our bounds recover the same asymptotic rates as~\citet{gao2024convergence}, despite their reliance on the stricter assumption of a strongly log-concave target (Proposition~\ref{prop:asymp_same}). For easier interpretation, we present a simplified error bound for the specific case where the forward SDE is the Ornstein–Uhlenbeck process (Theorem~\ref{thm:err_bound_OU}).
   
    \item We derive bounds on the initialization, discretization, and propagated score-matching error, which can in turn be used to develop heuristics for choosing hyperparameters such as the time scale, the step size used for discretization, and the acceptable score-matching error (see Table \ref{tab:examples}).
    
    \item  We study regime shifting to establish global convergence guarantees for the probability flow ODE in diffusion models (Proposition~\ref{prop:regime-shift}). This is crucial for a rigorous mathematical understanding of their sampling dynamics. Our analysis of this transition between noise- and data-dominated phases enables stronger, non-asymptotic convergence rates.

\end{enumerate}

\subsection{Related Work}

Existing studies of the convergence of trained score-based generative models (SGMs) invoke a variety of different distances. 
Total Variation (TV) distance and Kullback–Leibler (KL) divergence are the most commonly used in theoretical analyses~\citep{Sara2000,wainwright2019high}. 
For instance, theoretical guarantees for diffusion models 
in terms of TV or  KL have been studied in \cite{lee2022convergence,wibisono2022convergence, chen2022sampling,chen2023improved,chen2023probability,chen2023restoration, gentiloni2024theoretical, li2024towards, conforti2025kl}.
However, these metrics often fail to capture perceptual similarity in applications such as image generation. In contrast, the 2-Wasserstein distance is often preferred in practice, as it better reflects the underlying geometry of the data distribution. %
One of the most popular performance metrics for the quality of generated samples in image applications, the Fr\'echet inception distance (FID), measures the Wasserstein distance between the distributions of generated images and the distribution of real images~\citep{heusel2017}. Importantly, convergence in TV or KL does not generally imply convergence in Wasserstein distance unless strong conditions are satisfied~\citep{gibbs2002choosing}.

A smaller number of works go further to analyze convergence in Wasserstein distances, though these typically require additional assumptions like compact support or uniform moment bounds, see e.g.\ \citet{block2020generative,de2022convergence,lee2023convergence,gao2025wasserstein} for SDE-based samplers.
For example, \cite{gao2025wasserstein} propose non-asymptotic Wasserstein convergence guarantees for a broad class of SGMs assuming accurate score estimates and a smooth log-concave data distribution
(with unbounded support).
In general, the convergence rates are sensitive not only to the smoothness of the target distribution but also to the numerical discretization scheme and the regularity of the learned score.
Very recently,~\cite{beyler2025convergence} establish 2-Wasserstein convergence guarantees for diffusion-based generative models, treating both stochastic and deterministic sampling via early-stopping analysis.
Assuming the target distribution has bounded support ($X \in B(0,R)$ almost surely), they obtain bounds that grow exponentially with the support bound ($R$) and the inverse of the early stopping time ($1/\epsilon$), noting that this looseness stems from their minimal regularity assumptions.
Under stronger smoothness conditions ($X= Z+ \mathcal{N}(0,\tau I)$ with $Z \in B(0,R)$ and $\tau>0$ almost surely), they could improve the exponential dependence on the inverse of the early stopping time ($1/\epsilon$).
While very interesting, their results are limited to specific drift and diffusion coefficients and proposed rates are not tight.  
Further theoretical studies have been conducted on the theory of probability flow ODEs.
For example, \cite{gao2024convergence} established non-asymptotic convergence guarantees in 2-Wasserstein disctance for a broad class of probability flow ODEs, assuming the score function is learned accurately and the data distribution has a smooth and strongly log-concave density. %
However, the strong log-concavity assumption does not hold for many distributions of practical interest, including Gaussian mixture models.

Recently, there has been growing interest in relaxing the common assumption of strong log-concavity in the analysis of SGMs.
\cite{gentiloni2025beyond} derived 2-Wasserstein convergence guarantees for SGMs under \emph{weak log-concavity}, a milder assumption than strong log‐concavity. Exploiting the regularizing effect of the Ornstein–Uhlenbeck (OU) process, they show that weak log-concavity evolves into strong log-concavity via a PDE analysis of the forward process. Their analysis, specific to stochastic samplers and the OU process, identifies contractive and non‐contractive regimes and yields explicit bounds for settings such as Gaussian mixtures.
\cite{bruno2025wasserstein}  investigate whether SGMs can be guaranteed to converge in 2-Wasserstein distance when the data distribution is only semiconvex and the potential admits discontinuous gradients. However, their results are likewise restricted to stochastic samplers and the OU
process.
\cite{brigati2024heat} also proposed a different weakening of  log-concavity assumption, in the form of a Lipschitz perturbation of a log-concave distribution. 
This includes, in particular, measures which are log-concave outside some ball $B(0, R)$ while satisfying a weaker Hessian bound inside  $B(0, R)$.  Other forms of relaxation known as $F$-concavity have also been studied in \cite{ishige2024eventual}.
A key feature of these assumptions is the emergence of a \emph{regime shifting} behavior (also referred to as \emph{creation of log-concavity} or \emph{eventual log-concavity}), whereby the smoothing effect of the flow renders the distribution log-concave after some time. Much of the theoretical analysis in this paper builds on deriving quantitative controls over this phenomenon.

A recent alternative to diffusion models is flow matching, which learns vector fields over a family of intermediate distributions rather than the score function, offering a more general framework. Recent works have further investigated theoretical bounds for flow matching~\citep{albergo2022building,albergo2023stochastic}. However, these results either still rely on some form of stochasticity in the sampling procedure or do not apply to data distributions without full support.
\cite{benton2023error}  presents the first bounds on the error of the flow matching procedure that apply with fully deterministic sampling for data distributions without full support. 
Under regularity assumptions, \cite{benton2023error} show that the 2-Wasserstein distance between the approximated and the true density is bounded by the approximation error of the vector field 
and an exponential factor of the Lipschitz constant of the velocity. 
While interesting, their  bound is derived under the assumption of a continuous-time flow ODE, and does not account for discretization errors that occur in practice, for instance when employing numerical ODE solvers.
Also, their bound exhibits exponential growth with respect to the Lipschitz constant of the velocity, implying that highly nonlinear flows  may result in significantly weaker guarantees. 

Despite the growing body of literature, most existing convergence results—whether for stochastic or deterministic samplers—consider less suitable distance measures (in particular TV and KL), are derived under simplified settings (e.g.\ ignoring the discretization error), or, more importantly, rely on strong structural assumptions, such as log-concavity or bounded support of the data distribution. A substantial gap remains in understanding how the convergence rates for deterministic samplers change when those assumptions are weakened under a general setting of drift and diffusion coefficients.

\paragraph{Paper outline} %
Section~\ref{sec:prel} introduces SGMs, highlighting the approximations that are necessary to enable sampling from the probability flow ODE. 
In Section~\ref{sec:weak_conv}, we investigate the weak log-concavity assumption and establish its propagation in time as well as a regime shifting property, both of which are crucial for the proof of our error bound.
Section~\ref{sec:main_res} presents our main result, a non-asymptotic convergence bound for the 2-Wasserstein distance of the true and approximated sample distribution. We provide a result for the specific choice of the Ornstein-Uhlenbeck process, yielding a directly interpretable bound, and a general result that applies to any choice of the drift and diffusion function. Moreover, we compare our result to the one in \citet{gao2024convergence} imposing the stricter assumption of strong log-concavity of the data distribution, revealing the remarkable feature that the asymptotics remain the same.
Finally, in Section~\ref{sec:conclusion}, we summarize our results and provide an outlook into related future research directions.
Additional technical results and detailed proofs are provided in the Appendix.

\paragraph{Notation}\label{Notations}
For $a, b \in \mathbb R$, we write $a \wedge b$ as a shorthand for $\min\{a, b\}$ and $a \vee b$ for $\max \{a,b\}$ . 
Given a random variable $X \in \mathbb R^d$, we denote its law by $\mathcal L (X)$ and its 
$L_2$-norm as 
$\| X\|_{L_2} :=\sqrt{ \mathbb E (\|X\|^2)}$, where $\|\cdot \|$ is the Euclidean norm in $\mathbb R^d$. 
For any two probability measures $\mu, \nu \in \mathcal P_2(\mathbb R^d)$, the space of measures on $\mathbb R^d$ with finite second moment, the 2-Wasserstein distance, based on the Euclidean norm, is defined as 
\begin{equation}\label{eq:wass-def}
    \mathcal W_2(\mu, \nu) :=
\left ( \inf_{ X \sim \mu, Y \sim \nu   } \mathbb E \|X - Y\|^2 \right ) ^{\frac 12}\,,
\end{equation}
where the infimum is taken over all possible couplings of $\mu$ and $\nu$. 
\section{Preliminaries on score-based generative models}\label{sec:prel}
This section introduces SGMs and the their ODE-based implementation of the sampling process (probability flow ODE), which provides the framework for our analysis. 
Denote with $p_0 \in \mathcal P(\mathbb R^d)$ an unknown probability distribution on $\mathbb R^d$. Our goal is to generate new samples from $p_0$ given a data set of  independent and identically distributed observations. 
SGMs use a two-stage procedure to achieve this. 
First, noisy samples are progressively generated by means of a diffusion-type stochastic process. 
Then, in order to reverse this process, a model is trained to approximate the score, enabling the generation of new samples.

More concretely, noisy samples are generated from the forward process \(\{X_t\}_{t\in[0,T]}\), solution to the stochastic differential
equation (SDE)
\begin{equation}
  \mathrm{d}X_t = -f(t) X_t\,\mathrm{d}t \;+\; g(t)\,\mathrm{d}B_t,
  \qquad X_0 \sim p_0,                                          \label{eq:sde}
\end{equation}
where \(f,g:[0,T]\to\mathbb{R}_{\ge 0}\) are continuous and non-negative, $g(t)$ is positive for all $t > 0$, and \(B_t\) is a
standard \(d\)-dimensional Brownian motion.
Through this process, the unknown data
distribution  \(p_0\) progressively evolves over time into the family $\{p_t, \, t \geq 0\}$, where $p_t$ denotes the marginal law of the process $X_t$.  
The solution to \eqref{eq:sde}  is given by (see e.g.\ \citealp[Chapter 5.6]{karatzas2012brownian})
\begin{equation}
\label{eq:forw_SDE_solution}
    X_t=e^{-\int_{0}^{t}f(s) \de s}\,X_0+\int_{0}^{t}e^{-\int_{s}^{t}f(v) \de v}g(s) \de B_s .
\end{equation}
Note that the stochastic integral in \eqref{eq:forw_SDE_solution} has Gaussian distribution:
\[
\int_{0}^{t}e^{-\int_{s}^{t}f(v) \de v}g(s) \de B_s  \sim \mathcal{N} \left (0,\int_{0}^{t}e^{-2\int_{s}^{t}f(v) \de v}g^2(s) \de s \cdot I_d \right ) =: \hat p_t,
\]
independent of $p_0$.

Common instances used in score-based generative modeling
are \emph{variance-exploding} (VE) and \linebreak \emph{variance-preserving} (VP)
SDEs \citep{song2021score}. In a VE-SDE, we choose
\begin{equation} \label{eq:VE-SDE}
    f(t) \equiv 0 \quad\text{and}\quad g(t) = \sqrt{\frac{\de \qty[\sigma^2(t)]}{\de t}},
\end{equation}
whereas in a VP-SDE, it holds that
\begin{equation} \label{eq:VP-SDE}
    f(t) = \half \beta(t) \quad\text{and}\quad g(t) = \sqrt{\beta(t)}
\end{equation}
for some non-negative non-decreasing functions $\sigma(t)$ and $\beta(t)$, respectively.
The name ``variance‐preserving'' in the VP--setting can be justified by noting that noise is added in the forward process in a way that exactly offsets the drift's tendency to contract the variance.
Namely, $\int_0^T f(t) \de t $ diverges while
$\int_0^T e^{-2\int_t^T f(s) \de s} g^2(t) \de t \to 1 $ as $T\to \infty$. 
Therefore, in the VP-case $X_t$ has stationary distribution $p_\infty = \mathcal N(0, I_d)$. 

Next, score matching is performed, i.e.\ the unknown true score function $\nabla_x \log p_t$ is 
estimated by training a model in some family $\{s_\theta(t, x), \theta \in \Theta\}$, typically a deep neural network. 
This is achieved by minimizing a denoising score matching objective of the form \citep{song2019generative}
\begin{equation}
\mathcal{L}(\theta)
= \int_0^T
  \,\mathbb{E}
  \Big[\big\|\, s_\theta(X_t,t)
  \;-\; \nabla_{x}\log p_{t}(X_t) \big\|_2^2 \Big].
\label{eq:dsm-general}
\end{equation}

Practical implementations of \eqref{eq:dsm-general} typically introduce a time-dependent weighting function and rewrite the objective in terms of conditional expectations to make the optimization viable. These modifications do not affect our analysis; the only requirement is that a sufficiently accurate model is available (see Assumption~\ref{ass:score_error}).

The key idea behind SGMs is that the dynamics of the reverse process are explicitly characterized, allowing for new sample generation. In this work, we focus on the ODE formulation of this time reversal, namely the probability-flow ODE.
According to \citet{song2021score}, the time–reversed state 
\(\tilde{X}_t := X_{T-t}, \, t\in[0, T]\) satisfies the ordinary differential
equation
\begin{equation}
  \frac{\de \tilde{X}_t}{ \de t}
    = f(T-t)\tilde{X}_t
      + \frac{1}{2}\,g^2(T-t)\,\nabla \log p_{T-t}(\tilde{X}_t),
  \qquad \tilde{X}_0 \sim p_T,                                  \label{eq:flow}
\end{equation}
which is the so-called \textit{probability flow ODE} underpinning modern SGMs.

In the VP-case,
$\nabla \log p_\infty (x) = -x$, and the probability flow ODE can be rewritten as 
\begin{equation}\label{eq:flow-norm}
 \frac{\de \tilde X_t}{\de t} = \frac 12 \beta(T-t) \nabla \log \tilde p_{T-t} (\tilde X_t)\,,
\end{equation}
where $\tilde p_{t} = p_t / p_\infty$. 
The ``normalized" flow in  \eqref{eq:flow-norm} plays the role of an ODE equivalent of \cite[equations (5)--(7)]{gentiloni2025beyond}.

Three approximations are needed in order to use ODE \eqref{eq:flow} to create new samples in practice. First, note that the distribution $p_T$ of the final state $X_T$ is unknown. 
We therefore approximate it with a tractable law from which samples can be generated efficiently.
Following \cite{gao2024convergence}, we replace $p_T$ with 
$\hat p_T$ and consider the probability flow
\begin{equation}\label{eq:flow-hat}
  \frac{\de Y_t}{ \de t}
    = f(T-t)Y_t
          + \frac{1}{2}\, g^2(T-t)\,\nabla \log p_{T-t}(Y_t),
  \qquad Y_0 \sim \hat p_T.                                  
\end{equation}
The only difference between $X_t$ and $Y_t$ lies in their initial distribution. In the VP case, one might also start the reverse process from the invariant distribution $p_\infty$, i.e.\ $Y_0 \sim \mathcal N(0, I_d)$. 

Second, we employ a numerical discretization method to approximate the solution of ODE \eqref{eq:flow-hat}, as it is not generally available in closed form. Similarly to \cite{gao2024convergence}, we consider an exponential integrator discretization for this purpose. 
This method has been shown to be faster than other options such as Euler method or RK45, as it is more stable with respect to taking larger step sizes \citep{zhang2023fast}.
Specifically, the interval $[0,T]$ is split into discrete time steps $t_k = kh$ for $k \in \{0, 1, \dots, K\}$ and step size $h > 0$. Without loss of generality, we assume that $T=Kh$ for some positive integer $K$. On each interval $t_{k-1} \le t \le t_k$, ODE \eqref{eq:flow-hat} is then approximated by
\begin{equation}
\label{eq:flow_discr}
    \frac{\de \widehat{Y}_t}{\de t} = f(T-t)\widehat{Y}_t + \half g^2(T-t) \nabla \log p_{T-t_{k-1}} \qty(\widehat{Y}_{t_{k-1}}).
\end{equation}
Since the non-linear term is not dependent on $t$ anymore, this ODE can be explicitly solved on each interval, yielding
\begin{align*}
    \widehat{Y}_{t_k} &= e^{\intk f(T-t)\de t} \widehat{Y}_{t_{k-1}}\\
    &\qquad + \half p_{T-t_{k-1}} \qty(\widehat{Y}_{t_{k-1}}) \cdot \intk e^{\int_t^{t_k} f(T-s) \de s} g^2(T-s) \de t
\end{align*}
for $k \in \{1, \dots, K\}$. As in \eqref{eq:flow-hat}, the initial distribution is given by $\widehat{Y}_0 \sim \hat{p}_T$. 

Finally, since the score function $\nabla \log p_{t}$ is unknown in practice, we approximate it by the score model $s_\theta(x,t)$. 
This leads to an approximation of \eqref{eq:flow_discr} given by
\begin{equation}
\label{eq:flow_score}
    \frac{\de \widehat{Z}_t}{\de t} = f(T-t)\widehat{Z}_t + \half g^2(T-t) s_\theta \qty(\widehat{Z}_{t_{k-1}}, T-t_{k-1})
\end{equation}
with $\widehat{Z}_0 \sim \hat{p}_T$ and solution
\begin{align*}
    \widehat{Z}_{t_k} &= e^{\intk f(T-t)\de t} \widehat{Z}_{t_{k-1}}\\
    &\qquad + \half s_\theta \qty(\widehat{Z}_{t_{k-1}}, T-t_{k-1}) \cdot \intk e^{\int_t^{t_{k}} f(T-s) \de s} g^2(T-s) \de t
\end{align*}
for $k \in \{1, \dots, K\}$. 

This means that, effectively\textemdash after replacing the initial distribution, learning the score, and discretizing\textemdash one is able to sample from the 
law  $\mathcal{L}(\widehat{Z}_{t_{K}})$,
which serves as a viable approximation of the unknown data distribution $p_0$. 
Our objective is then to quantify the accuracy of the method by providing bounds on the 2-Wasserstein distance between the generated samples and the target distribution
$\mathcal W_2 \big (\mathcal{L}(\widehat{Z}_{t_{K}}), p_0 \big )$.
A first brief summary of our results is given in Table \ref{tab:summary}.

\begin{table}
	\caption{In our main result (Theorem~\ref{thm:main_thm}), we show that the error $\mathcal W_2 (\mathcal{L}(\widehat{Z}_{t_{K}}), p_0 )$ can be bounded by the sum of three error components $E_0$, $E_1$, and $E_2$. The table provides a summary of the main properties of these terms and their specific heuristics in the specific case of the OU process, i.e.\ for $f \equiv 1$ and $g \equiv \sqrt{2}$, indicated by an $\text{asterisk}^\ast$ (see Theorem~\ref{thm:err_bound_OU}).}
	\label{tab:summary}
    \renewcommand{\arraystretch}{2}
    \vspace*{0.5\baselineskip}
    \centering
    \begin{tabular}{c|c|c|c}
        &$E_0(f,g,T)$ &$E_1(f,g,K,h)$ &$E_2(f,g,K,h,\mathcal{E})$\\
        \hline
        Error source &Initialization &Discretization &Score matching\\
        Vanishes with &$T \to \infty$ &$h \to 0$ &$\mathcal{E} \to 0$\\
        $\text{OU process}^\ast$ & $\mathcal{O}\qty(e^{-T}\sqrt{d})$ & $\mathcal{O}\qty(e^{Th} Th \qty(\sqrt{d} + T))$ & $\mathcal{O}\qty(e^{Th} T \mathcal{E})$\\
        Error $\le \varepsilon$ $\text{if}^\ast$ &$T \ge \mathcal{O}\qty(\log \qty(\frac{\sqrt{d}}{\varepsilon}))$ & $h \le \mathcal{O}\qty(\frac{\varepsilon}{\sqrt{d} \log\qty(\frac{\sqrt{d}}{\varepsilon})})$
        & $\mathcal{E} \le \mathcal{O}\qty(\frac{\varepsilon}{\log\qty(\frac{\sqrt{d}}{\varepsilon})})$
    \end{tabular}
\end{table}

\section{Weak concavity}
\label{sec:weak_conv}

Our main result establishes an error bound for the probability flow ODE, relying on a weaker assumption than strong log-concavity of the density $p_0$. In particular, we use the notion of weak concavity which was also used in \cite{gentiloni2025beyond} to derive a convergence result for the specific case of $f(t) = 1$ and $g(t) = \sqrt{2}$ resulting in the Ornstein-Uhlenbeck process. It is defined as follows.

\begin{definition}[Weak convexity]
\label{def:conv_profile}
    The weak convexity profile of a function $g \in C^1(\R^d)$ is defined as
    $$ \kappa_g(r) = \inf_{x, y \in \R : \norm{x-y} = r} \left\{ \frac{\dotprod{\nabla g(x) - \nabla g(y)}{x-y}}{\norm{x-y}^2} \right\}, \quad r>0.$$
    We say that $g$ is $(\alpha, M)$-weakly convex if 
    $$ \kappa_g(r) \ge \alpha - \frac{1}{r} f_M(r) \quad \text{for all } \, r > 0$$
    for some constants $\alpha, M > 0$ and 
    $$ f_M(r) = 2\sqrt{M} \tanh \qty(\half \sqrt{M} r). $$
    Moreover, we say that $g$ is $(\alpha, M)$-weakly concave if $-g$ is $(\alpha, M)$-weakly convex.
\end{definition}
\noindent The weak convexity assumption means that the function is approximately convex at ``large scales'' (large $r$), while allowing small non-convex fluctuations at short distances (small $r$). 
Importantly, $(\alpha, M)$-weak concavity implies $(\alpha-M)$-strong concavity if $\alpha-M > 0$, as laid out in Lemma~\ref{lem:weak_implies_vague}, meaning that it is in fact a more general assumption. A relevant example for a family of distributions that are weakly but not strongly log-concave are Gaussian mixture models \citep[Proposition~4.1]{gentiloni2025beyond}. A specific example of such a mixture model including graphs of the log-density and score function are given in Example \ref{ex:gmm} in Appendix~\ref{sec:app_weak_conc}. Note that, due to their strong log-concavity at large scales, weakly log-concave distributions necessarily need to have sub-gaussian tails. This means that any distribution that is not sub-gaussian, such as the Laplace distribution, cannot be weakly log-concave. 
This naturally rises the question if there exist distributions that are sub-gaussian but not weakly log-concave. The answer to this question is positive. In Example~\ref{ex:non_weak_concave} in Appendix~\ref{sec:app_weak_conc}, we construct a corresponding example. The main issue is that the score exhibits an excessively steep increase at one point.

\begin{remark}[General $f_M(r)$]
    As stated by~\citet[Theorem~5.4]{conforti2023projected}, a general class for $f_M$ is possible, provided that $f_M \in \widehat{\mathcal{G}} := \left\{ g \in \mathcal{G} ~\text{such that}~ g' \ge 0,\; 2g'' + g g' \le 0 \right\}$, where
    \begin{equation*} 
        \mathcal{G} := \left\{ g \in \mathcal{C}^2\bigl((0, \infty), \mathbb{R}_+\bigr) : r \mapsto r^{1/2} g(r^{1/2})~\text{is non-decreasing and concave, and}~\lim_{r \downarrow 0} r g(r) = 0 \right\}.
    \end{equation*}
    We also need that there exists an $M > 0$ such that $rg(r) \le M r^2$ in order for the second part of Lemma~\ref{lem:weak_implies_vague} to hold.
    Naively speaking, the set $\widehat{\mathcal{G}}$ consists of smooth, non-negative, non-decreasing functions $g(r)$ defined on $(0, \infty)$ that grow in a controlled way and do not bend upward too rapidly. The transformation $r \mapsto r^{1/2} g(r^{1/2})$ must be non-decreasing and concave, ensuring mild growth behavior. The condition $2g'' + g g' \le 0$ further constrains how sharply the function is allowed to curve upward.  
\end{remark}

In the following, we investigate the concavity (and Lipschitz smoothness) of $\log(p_t)$ given that $\log(p_0)$ is weakly log-concave (and Lipschitz smooth). In other words, we establish results on how the weak concavity and Lipschitz assumptions propagate through time following the forward SDE \eqref{eq:sde}. Our main result heavily relies on these findings.

\subsection{Propagation in time of weak log-concavity}
\label{sec:weak_conv_prop}

The following Proposition~shows that, if $p_0$ is weakly log-concave, this property is preserved by $p_t$.

\begin{proposition}[Propagation of weak log-concavity in time]
\label{prop:weak_pt}
    If $p_0$ is $(\alpha_0, M_0)$-weakly log-concave, then $p_t$ is $(\alpha(t), M(t))$-weakly log-concave with 
    \begin{equation} \label{eq:alpha_def}
        \alpha(t) = \frac{1}{\frac{1}{\alpha_0} e^{-2\int_0^t f(s)ds} + \int_0^t e^{-2\int_s^t f(v)dv} g^2(s) ds}
    \end{equation}
    and
    \begin{equation} \label{eq:M_def}
        M(t) = \frac{M_0 e^{2\int_0^t f(s)ds}}{\left(1 + \alpha_0 \int_0^t e^{2\int_0^s f(v)dv} g^2(s) ds\right)^2}.
    \end{equation}
    This implies in particular that
    \[
    \langle \nabla \log p_t (x) - \nabla \log p_t (y), x - y \rangle \leq - (\alpha(t) - M(t)) \| x - y \|^2 .
    \]
\end{proposition}

\noindent Note that this is a generalization of the result in \citet[Equation (5.4)]{gao2025wasserstein} since $\alpha(t) = a(t)$ and $M(t) = 0$ if and only if $M_0 = 0$. 

\paragraph{Regime shifting} %
An interesting property of the forward flow is that the law $p_t$ becomes strongly log-concave after a finite amount of time,  even if $p_0$ is only weakly log-concave. We call this the \emph{regime shift} property.
It plays a central role in establishing convergence guarantees of the probability flow, see Proposition~\ref{prop:yt-p0-bound} below.

The forthcoming Proposition~\ref{prop:regime-shift} formalizes the regime shift property of our model. 
Intuitively, it states that, if $\alpha_0 -  M_0 > 0$, i.e.\ if $p_0$ is strongly log-concave, then 
$p_t$ is guaranteed to remain strongly log-concave.
Otherwise, if $\alpha_0 -  M_0 \le 0$, we have a regime shift result, and we are able to explicitly quantify the time at which this change takes place.
This is compatible with what has been observed in the literature for OU forward processes \citep{gentiloni2025beyond}. 
Let 
\begin{equation}\label{eq:tau_def}
    \tau(\alpha, M) \coloneqq
    \left \{
    \begin{aligned}
        & 0, & \alpha -  M > 0 \\
       &  \inf \left \{ t > 0 : \int_0^{t} e^{2\int_0^{s} f(v) \de v} g^2(s) \de s > \frac{M - \alpha}{\alpha^2}  \right \},  & \alpha -  M \leq  0
    \end{aligned}
    \right. 
\end{equation}
for $\alpha, M \in \mathbb R$. 
Since the integral in the inequality above is strictly increasing, we have
$\tau(\alpha, M) < \infty$. 

\begin{proposition}[Regime shifting]\label{prop:regime-shift}
    For $ 0 < t < T$, it holds that
    \begin{equation*}
        \begin{cases}
            p_t \, \text{ is weakly log-concave}, & t \in (0, \tau(\alpha_0, M_0) \wedge T)
            \\
            p_t \, \text{ is strongly  log-concave}, & t \in [\tau(\alpha_0, M_0) \wedge T , T)\,.
        \end{cases}
    \end{equation*}
\end{proposition}    
\noindent For example, for the Ornstein-Uhlenbeck process, 
\begin{equation}\label{eq:tau-ou}
    \tau(\alpha_0, M_0) = \log \sqrt{\frac{\alpha_0^2 + M_0 - \alpha_0}{\alpha_0^2}}\,, 
\end{equation}
which matches \citet[equation (26)]{gentiloni2025beyond}.
For a derivation of equation~\eqref{eq:tau-ou}, we refer to 
Example~\ref{ex:kt_vp} in Appendix~\ref{sec:app_weak_conc}, where formulas for the general VP case are presented.

The weak (log-)concavity constant $K(t) \coloneqq \alpha(t) - M(t)$ being negative for $t=0$ and becoming positive for $t = \tau(\alpha_0, M_0)$ rises the question whether this transition progresses monotonously. This is, in fact, not necessarily the case.
See Figure \ref{fig:alphaM} in Appendix~\ref{sec:app_weak_conc} for a graphical representation of possible behaviors.

\subsection{Propagation in time of Lipschitz continuity}

Assuming weak log-concavity of $p_0$ also guarantees Lipschitz continuity of the score function $\nabla \log(p_0)$ to propagate through the forward SDE \eqref{eq:forw_SDE_solution} as the following result shows.

\begin{proposition}[Propagation of Lipschitz continuity in time]
\label{prop:weak_Lip}
    If $p_0$ is $(\alpha_0, M_0)$-weakly log-concave and $\nabla \log p_0$ is $L_0$-Lipschitz continuous, i.e.
    \begin{equation*}
        \norm{\nabla \log p_0(x) - \nabla \log p_0(y)} \le L_0 \norm{x-y},
    \end{equation*}
    then $\nabla \log p_t$ is $L(t)$-Lipschitz continous, i.e.
    \begin{equation*}
        \norm{\nabla \log p_{t}(x) - \nabla \log p_{t}(y)} \le L(t) \norm{x-y}
    \end{equation*}
    with 
    \begin{equation} \label{eq:L_def}
        L(t) = \max\qty{\min\qty{\qty(\int_0^t e^{-2\int_s^t f(v)dv} g^2(s) ds)^{-1}, e^{2\int_0^t f(s)ds} L_0}, -\Big(\alpha(t) - M(t)\Big)}.
    \end{equation}
\end{proposition}

\noindent This is a proper generalization of a corresponding result for strongly log-concave distributions $p_0$ given in \citet[Lemma~9]{gao2025wasserstein}, as, in that case, $\alpha(t)- M(t) > 0$ for all $t \in [0,T]$ and the maximum in \eqref{eq:L_def} is achieved at the first term matching the definition of $L(t)$ in \citet{gao2025wasserstein}. 
\section{Main result}
\label{sec:main_res}

This section presents our main result, a non-asymptotic error 
bound for the approximated probability flow \eqref{eq:flow_score}. There are three sources of error according to the approximations of the probability flow ODE \eqref{eq:flow} explained in Section~\ref{sec:prel}. The first one, the initialization error, caused by using $Y_0 \sim \hat{p}_T$ instead of $p_T$, see \eqref{eq:flow-hat}, can be reduced by choosing a large time scale $T$. The second error source resulting from the numerical discretization $\widehat{Y}_t$ of the ODE as given in \eqref{eq:flow_discr}, can be alleviated by a small step size $h$. Lastly, the score-matching error, i.e.\ the distance between the true score $\nabla \log p_t(x)$ and its estimated counterpart $s_\theta(x,t)$, needs to be controlled in order for $\widehat{Z}_t$ as defined in \eqref{eq:flow_score} to be close to $\widehat{Y}_t$. Our non-asymptotic error bound accounting for all three of these approximations can be used to derive heuristics for how to choose the time scale $T$, the step size $h$, and the admissible score-matching error, say $\mathcal{E}$, in practical applications. Note that, as opposed to $T$ and $h$, the admissible score-matching error $\mathcal{E}$ cannot be directly chosen, but rather determines how to pick $s_\theta(x,t)$. When using a neural network, for example, $\mathcal{E}$ might affect its architecture, the number of epochs used for training, and the necessary number of training samples. %
In order for our error bound to hold, we impose the following assumptions.

\begin{assumption}[Regularity of the target]
\label{ass:weak_concave_Lip}
    The density of the data distribution $p_0$ is twice differentiable and positive everywhere. Moreover, $\nabla \log p_0$ is $(\alpha_0, M_0)$-weakly concave in the sense of Definition \ref{def:conv_profile} as well as $L_0$-Lipschitz continuous, meaning that for all $x,y \in \R^d$, it holds that
    \begin{equation*}
        \norm{\nabla \log p_0(x) - \nabla \log p_0(y)} \le L_0 \norm{x-y}.
    \end{equation*}
\end{assumption}
\noindent The first part of Assumption~\ref{ass:weak_concave_Lip} has been employed in previous works such as~\cite{gentiloni2025beyond}. Notably, it is a relaxed version of strong log-concavity which is the prevailing assumption in related works, e.g.~\citet{bruno2023diffusion, li2022sqrt, gao2024convergence, gao2025wasserstein}.
The second part, i.e.\ the Lipschitz continuity of the score function, is a standard regularity condition that ensures the gradient of the log-density varies smoothly and is also considered in a large number of previous works, for example, \cite{chen2023improved, gao2024convergence, taheri2025regularization, gao2025wasserstein}. In particular, \citet[Proposition~4.1]{gentiloni2025beyond} shows that Gaussian mixtures satisfy both the weak log-concavity and log-Lipschitz conditions, highlighting the broad applicability of this assumption.
\begin{assumption}[Lipschitz continuity in time]
\label{ass:Lip_in_time}
    There exists some $L_1 > 0$ such that for all $x \in \R^d$
    \begin{equation*}
        \sup_{\substack{k \in \{1, \dots, K\}\\t_{k-1} \le t \le t_k}} \norm{\nabla \log p_{T-t}(x) - \nabla \log p_{T-t_{k-1}}(x)} \le L_1 h (1+\norm{x}).
    \end{equation*}
\end{assumption}

\noindent Assumption~\ref{ass:Lip_in_time} imposes a Lipschitz condition on the score function with respect to time, ensuring that the scores vary smoothly over time. This assumption is mainly employed to bound the discretization error (see proof of Proposition~\ref{prop:yt-zt-bound}) and has been invoked widely~\citep{gao2024convergence,gao2025wasserstein}. 
A straightforward motivation is the idealized setting $X_0 \sim \mathcal{N}(0, \sigma^2 I_d)$, in which case its validity has been shown in \citet[p.\  8-9]{gao2025wasserstein}.

\begin{assumption}[Score-matching error]
\label{ass:score_error}
    There exists some $\mathcal{E} > 0$ such that
    \begin{equation*}
        \sup_{k \in \{1, \dots, K\}} \norm{\nabla \log p_{T-t_{k-1}}\qty(\widehat{Z}_{t_{k-1}}) - s_\theta\qty(\widehat{Z}_{t_{k-1}}, T-t_{k-1})}_{L_2} \le \mathcal{E}.
    \end{equation*}
\end{assumption}

\noindent Assumption~\ref{ass:score_error} ensures the accuracy of the learned score function. Just as in similar papers on the topic~\citep{gao2024convergence, gao2025wasserstein, gentiloni2025beyond},
it allows us to separate the convergence properties of the sampling algorithm from the challenges of score estimation. Our work focuses on the algorithmic aspects under idealized score estimates; the statistical error due to learning the score from data is the subject of another rich line of research~\citep{zhang2024minimax, wibisono2024optimal, dou2024optimal}.

\subsection{Error bound for the Ornstein-Uhlenbeck process}
\label{sec:main_OU}

Since our main result, a general error bound accounting for all possible functions $f$ and $g$, is rather complex and does not allow for a direct translation into a lower bound for $T$ and upper bounds for $h$ and $\mathcal{E}$, we first consider a specific case that is readily interpretable and then turn to the general case.

\begin{theorem}[Error bound for the OU process] \label{thm:err_bound_OU}
    For the Ornstein-Uhlenbeck process, i.e.\ $f(t) \equiv 1$ and $g(t) \equiv \sqrt{2}$, it holds that
    \begin{equation*}
        \mathcal{W}_2(\mathcal{L}(\widehat{Z}_T), p_0) \le \mathcal{O}\Big( \underbrace{e^{-T} \norm{X_0}_{L_2}}_{\textup{Initialization error}} + \underbrace{e^{Th}Th(\norm{X_0}_{L_2} + \sqrt{d} + T)}_{\textup{Discretization error}} + \underbrace{e^{Th}T\mathcal{E}}_{\textup{Propagated score-matching error}} \Big).
    \end{equation*}
\end{theorem}
\noindent The proof of this result is provided in Appendix~\ref{sec:app_OU}.
The theorem implies that, in order to achieve a given accuracy level $\varepsilon$, meaning that $\mathcal W_2(\mathcal{L}(\widehat{Z}_T), p_0) \le \varepsilon$, we need
\begin{enumerate}
    \item the time scale $T$ to be large enough for the initialization error to be small, in particular
    \begin{equation*}
        T \ge \mathcal{O}\qty(\log(\frac{\norm{X_0}_{L_2}}{\varepsilon})),
    \end{equation*}
    
    \item the step size $h$ to be small enough for the discretization error to be small, in particular
    \begin{equation*}
        h \le \mathcal{O}\qty(\frac{\varepsilon}{T \qty(\norm{X_0}_{L_2} + \sqrt{d})}) \le \mathcal{O}\qty(\frac{\varepsilon}{\log(\varepsilon^{-1}\norm{X_0}_{L_2}) \qty(\norm{X_0}_{L_2} + \sqrt{d})}),
    \end{equation*}

    \item the score-matching error $\mathcal{E}$ to be small enough for the propagated score-matching error to be small, in particular
    \begin{equation*}
        \mathcal{E} \le \mathcal{O}\qty(\frac{\varepsilon}{T}) = \mathcal{O}\qty(\frac{\varepsilon}{\log(\varepsilon^{-1}\norm{X_0}_{L_2})}).
    \end{equation*}
\end{enumerate}
If $\norm{X_0}_{L_2} = \mathcal{O}(\sqrt{d})$ as it is the case when $p_0$ is strongly log-concave, these complexities coincide with those in \citet[Table 1]{gao2024convergence} after translating the lower bound for $T$ to a bound for $K = T/h$. This is remarkable as our results do not assume strong concavity of the data distribution and thus account for more general settings. In fact, this finding is not specific to the OU process but applies to all other VP and also VE SDEs considered by Gao and Zhu, as we will show in Section~\ref{sec:asymptotics}.

\subsection{Error bound for general \textit{f} and \textit{g}}

Now, we state the error bound for general functions $f$ and $g$. Its proof is provided in Section~\ref{sec:proof}.

\begingroup
\allowdisplaybreaks
\begin{theorem}[Error bound for the probability flow ODE]
\label{thm:main_thm}
    Under Assumptions \ref{ass:weak_concave_Lip}, \ref{ass:Lip_in_time}, and \ref{ass:score_error}, 
    it holds that
    \begin{equation*}
        \mathcal W_2\qty(\law{\widehat{Z}_T}, p_0) \leq 
        \underbrace{E_0(f, g, T)}_{\textit{Initialization error}}
        + \underbrace{E_1(f, g, K, h)}_{\textit{Discretization error}} + \underbrace{E_2(f, g, K, h, \mathcal{E})}_{\textit{Propagated score-matching error}},
    \end{equation*}
    where
    \begin{align}
        E_0(f, g, T) &\coloneqq C(\alpha_0,  M_0) e^{- \half \int_0^T g^2(t)|\alpha(t) - M(t)| \de t} \|X_0\|_{L_2}, \label{eq:E0}\\
        E_1(f, g, K, h) &\coloneqq \sum_{k=1}^K \qty(\prod_{j=k+1}^K \gamma_{j,h}) e^{\int_{t_k}^{T} f(T-t)\de t} \nonumber\\
        &\qquad\quad \cdot \left(\half L_1 h (1 + \theta(T) + \omega(T)) \intk e^{\int_t^{t_k} f(T-s) \de s} g^2(T-t) \de t \right. \nonumber\\
        &\qquad\qquad \left. + \half \sqrt{h} \nu_{k, h} \qty(\intk \qty[e^{\int_t^{t_k} f(T-s) \de s} g^2(T-t) L(T-t)]^2 \de t)^\half\right), \label{eq:E1}\\
        E_2(f, g, K, h, \mathcal{E}) &\coloneqq \sum_{k=1}^K \qty(\prod_{j=k+1}^K \gamma_{j,h}) e^{\int_{t_k}^{T} f(T-t)\de t} \qty(\half \mathcal{E} \intk e^{\int_t^{t_k} f(T-s) \de s} g^2(T-t) \de t), \label{eq:E2}
    \end{align}
    the functions $\alpha(t)$, $M(t)$, $\tau(\alpha, M)$, and $L(t)$ are defined in \eqref{eq:alpha_def}, \eqref{eq:M_def}, \eqref{eq:tau_def}, and \eqref{eq:L_def}, respectively, and
    \begin{align}
        C(\alpha_0, M_0) &\coloneqq \exp \left( \frac{|\alpha_0 - M_0|}{\alpha_0^2 \wedge 1} \xi(\tau(\alpha_0, M_0)) \int_0^{\tau(\alpha_0, M_0)} g^2(t) \de t \right), \label{eq:C(alpha, M)}\\
        \xi(T) &\coloneqq \sup_{0\leq t \leq T} \min \left\{ e^{ 2 \int_0^{t} f(s) ds}, \frac{ e^{ 2 \int_0^{t} f(s) ds}}{(\int_0^{t} e^{2\int_0^{s} f(v) \de v} g^2(s) \de s)^2} \right\}, \label{eq:xi(T)} \\
        \gamma_{k,h} &\coloneqq 1 - \intk \delta_k(T-t) \de t + \half L_1 h \intk g^2(T-t) \de t, \label{eq:gamma_kh}\\
        \delta_k(T-t) &\coloneqq \half e^{-\int_{t_{k-1}}^t f(T-s) \de s} g^2(T-t) \big(\alpha(T-t) - M(T-t)\big) - \frac{1}{8} h g^4(T-t) L^2(T-t), \label{eq:delta_k}\\
        \theta(T) &\coloneqq \sup_{0 \le t \le T} e^{-\half \int_0^t g^2(T-s)(\alpha(T-s)-M(T-s)) - 2f(T-s) \de s} e^{- \int_0^T f(s) \de s} \norm{X_0}_{L_2}, \label{eq:theta(T)}\\
        \omega(T) &\coloneqq \sup_{0 \le t \le T} \qty(e^{-2 \int_0^t f(s)\de s} \norm{X_0}_{L_2}^2 + d \int_0^t e^{-2\int_s^t f(v)\de v} g^2(s) \de s)^\half, \label{eq:omega(T)}\\
        \nu_{k,h} &\coloneqq (\theta(T) + \omega(T)) \intk \qty[f(T-s) + \half g^2(T-s) L(T-s)] \de s, \nonumber\\
        &\qquad + (L_1(T+h) + \norm{\nabla \log p_0(\boldsymbol{0})}) \intk \half g^2(T-s) \de s. \label{eq:nu_kh}
    \end{align}
\end{theorem}
\endgroup

\noindent Note that the error terms $E_0$, $E_1$, and $E_2$ also depend on the weak concavity and Lipschitz constants $\alpha_0$, $M_0$, $L_0$ and, $L_1$ from Assumptions \ref{ass:weak_concave_Lip} and \ref{ass:Lip_in_time}. However, since these are determined by the data distribution $p_0$ and thus cannot be controlled by the user, we do not explicitly include them in the arguments.

Although the error bound in Theorem~\ref{thm:main_thm} looks rather complex, we can identify its key properties as follows. 
According to~\eqref{eq:E0}, $E_0$ depends on the drift $f$, the diffusion coefficient $g$, and the time horizon $T$. 
It decreases exponentially with $T$ and increases with factors related to the target distribution, namely $\alpha_0$, $M_0$, and $\|X_0\|_{L_2}$.
Thus, in practice, for sufficiently large $T$, the error $E_0$ can be neglected.
As stated in~\eqref{eq:E1}, $E_1$ depends on $f$, $g$, $K$, and also on the step size $h$. 
At its core lies a product over $\gamma_{j,h}$. Depending on the regime shift, each $\gamma_{j,h}$ takes values either less than or greater than one (see Proposition~\ref{prop:gamma_khT} in Appendix~\ref{sec:app_interpret}). 
A sufficiently small step size $h$ is necessary to control that product when the factors exceed one.
In particular, $E_1$ vanishes as $h$ goes to zero, which matches with intuition as it corresponds to the discretization error. 
Note that it increases with the Lipschitz constant of the target $L_1$,  $\|X_0\|_{L_2}$, and the dimensionality of the data $d$ (we refer to~\cite{taheri2025regularization}, who employ regularization techniques to reduce $d$ to a much smaller sparsity level  for diffusion models).
Finally, the propagated score-matching error $E_2$, defined in \eqref{eq:E2}, depends on $f$, $g$, $K$, $h$, and additionally on the score-matching error $\mathcal{E}$. 
It also involves the product over $\gamma_{j,h}$, as in $E_1$.
As $\mathcal{E} \to 0$, this error vanishes. 
Thus, to prevent this source of error from blowing up, the score-matching error $\mathcal{E}$ must be sufficiently small.
For a closer understanding of how large the time horizon $T$ and how small the score-matching error $\mathcal{E}$ and step size $h$ need to be, see the discussion following Theorem~\ref{thm:err_bound_OU} for the OU case, and Section~\ref{sec:asymptotics} for other VE and VP SDEs.

\subsection{Comparison to the strongly log-concave case}
\label{sec:comp_strong}

It is instructive to compare our result to the strongly log-concave case analyzed in \cite{gao2024convergence}.
In particular, Theorem~\ref{thm:main_thm} matches their Theorem~2 in case $p_0$ is strongly log-concave, i.e.\ $M_0=0$. To see that, note that our result differs from Gao and Zhu's in the following ways:
\begin{enumerate}
    \item In the initialization error, we have the additional coefficient $C(\alpha_0, M_0)$ as well as $\abs{\alpha(t) - M(t)}$ instead of $a(t)$ in the exponent. If $p_0$ is strongly log-concave, then $\tau(\alpha_0, M_0) = 0$ and thus $\xi(\tau(\alpha_0, M_0)) = 0$ implying that $C(\alpha_0, M_0) = 1$. Moreover, from the definitions in Proposition~\ref{prop:weak_pt}, it can be seen that, if $M_0 = 0$, then $M(t) = 0$ and $\alpha(t)$ equals $a(t)$ defined in \citet[equation (49)]{gao2024convergence} which is positive for all $t \in [0,T]$.
    \item In $\delta_k(T-t)$ and $\theta(T)$, the strong log-concavity parameter $a(T-t)$ of $p_{T-t}$ is naturally replaced by the weak log-concavity parameter $(\alpha(T-t) - M(T-t))$. As explained above, we have $\alpha(T-t) = a(T-t)$ and $M(T-t) = 0$ in case $p_0$ is strongly log-concave.
    \item The definition of the Lipschitz constant $L(t)$ of $p_t$ in Proposition~\ref{prop:weak_Lip} resembles the one in \citet[equation (27)]{gao2024convergence} but involves the additional term $-\big(\alpha(t) - M(t)\big)$. If $p_0$ is strongly log-concave, we have $\tau(\alpha_0, M_0) = 0$ and thus $\alpha(t) - M(t) > 0$ for all $t \in [0,T]$. Since the minimum in the definition \eqref{eq:L_def} of $L(t)$ is always non-negative, the additional term can be disregarded and the two definitions coincide.
    \item The coefficient in front of the second summand of $\delta_k(T-t)$ is $\frac{1}{8}$ instead of $\frac{1}{4}$. Note that this is better in the sense that it yields a tighter error bound. %
    \item The definition of $\nu_{k,h}$ involves the coefficient $T+h$ instead of $T$. We believe that the same should apply to Gao and Zhu's result, correcting \citet[equation (72)]{gao2024convergence} as illustrated in equation \eqref{eq:bound_p_T-s(0)} in the proof of Lemma~\ref{lem:nu_kh} in Appendix~\ref{sec:app_proof}.
    \item In the first summand of the discretization error $E_1(f, g, K, h)$, the coefficient $\norm{X_0}_{L_2}$ is replaced by $\theta(T)$. According to Lemma~\ref{lem:theta_bound} in Appendix~\ref{sec:app_interpret}, it holds that 
    \begin{equation*}
        \theta(T) \le \sqrt{C(\alpha_0, M_0)} \norm{X_0}_{L_2}^2.
    \end{equation*}
    In the strongly log-concave case, we have $C(\alpha_0, M_0) = 1$ as explained under point 1. Hence, in this case, $\theta(T) \le \norm{X_0}_{L_2}$, which is used in \citet{gao2024convergence}.
\end{enumerate}

Analyzing the effects of these differences on the asymptotic behavior of the error bound in case $p_0$ is weakly log-concave leads to the following result. Its proof is given in Appendix~\ref{sec:app_interpret}.

\begin{proposition}[Comparison to the strongly log-concave case] %
\label{prop:asymp_same}
    For any choice of $f$ and $g$ according to a VP-SDE \eqref{eq:VE-SDE} or VE-SDE \eqref{eq:VP-SDE}, the following holds.
    Even if $p_0$ is only weakly log-concave, the asymptotics of the error bound in Theorem~\ref{thm:main_thm} with respect to $T$, $h$, and $\mathcal{E}$ are the same as for the bound given in \citet[Theorem~2]{gao2024convergence}, which relies on the stricter assumption of strong log-concavity.
\end{proposition}

\noindent 
This is a striking result: the error $\mathcal{W}_2(\mathcal{L}(\hat{Z}_T), p_0)$ scales in $T$, $h$, and $\mathcal{E}$ exactly as under the more restrictive strong log-concavity assumption. This means, in particular, that the heuristics for choosing these hyperparameters remain exactly the same. We will provide more details on this matter in the following section.

\subsection{Guidelines for the choice of hyperparameters}
\label{sec:asymptotics}

Theorem~\ref{thm:err_bound_OU} treats the special case of $f \equiv 1$ and $g \equiv \sqrt{2}$ corresponding to the OU process. Many quantities simplified in this case, enabling us to derive explicit heuristics for how to choose the hyperparameters $T$, $h$, and $\mathcal{E}$ in order for the sampling error, measured in 2-Wasserstein distance, to be appropriately bounded. Now, we want to conduct a similar analysis for other choices of $f$ and $g$. 
Since only the asymptotics of the error bound are relevant for this purpose, and, according to Proposition~\ref{prop:asymp_same}, they match those of the strongly log-concave case, we do not have to derive the heuristics from scratch but can reuse the results from \citet[Section~3.3]{gao2024convergence}. 

Note that Gao and Zhu also make use of the fact that $\norm{X_0}_{L_2} = \mathcal{O}(\sqrt{d})$, which may not always apply when $p_0$ is only assumed to be weakly log-concave. Consequently, our bounds will involve an additional dependency on this term (as in Theorem~\ref{thm:err_bound_OU}). However, it seems natural to assume that the $L_2$-norm of $X_0$ scales with the dimension in this way as
\begin{equation*}
    \norm{X_0}_{L_2}^2 = \E{\norm{X_0}^2} = \norm{\mu_0}^2 + \text{tr}\qty(\Sigma_0) = \sum_{i=1}^d \qty(\mu_{0}^{(i)})^2 + \sum_{i=1}^d \Sigma_0^{(i,i)},
\end{equation*}
where $\mu_0 = (\mu_0^{(1)}, \dots, \mu_0^{(d)})^\top \in \R^d$ and $\Sigma_0 = (\Sigma_0^{(i,j)})_{i,j=1}^d \in \R^{d \times d}$ denote the mean and covariance matrix corresponding to $p_0$. Accordingly, $\norm{X_0}_{L_2} = \mathcal{O}(\sqrt{d})$ holds if the entries of $\mu_0$ and $\Sigma_0$ do not scale with the dimension $d$.

Table \ref{tab:examples} presents the heuristics for how to choose the time scale $T$, step size $h$, and acceptable score-matching error $\mathcal{E}$ in order to guarantee the error to be bounded by some small $\eps > 0$. It was directly derived from \citet[Table 1]{gao2024convergence}, translating the bounds for the number of steps $K$ to bounds for $T$. Note that we assume that $\norm{X_0}_{L_2} = \mathcal{O}(\sqrt{d})$ for the table to be applicable. We want to emphasize that this is not a limiting assumption as we can derive analogous results in case this condition is not met. Similar to the bounds for the OU process, given in Section~\ref{sec:main_OU}, this would entail the term $\norm{X_0}_{L_2}$ arising in the heuristics for $T$ and $h$. To keep the results simple, and because the assumption seems natural as argued above, we decided to not explicitly state this dependence in the table. For a derivation of the heuristics in Table \ref{tab:examples}, we refer to \citet[Corollaries 6-9]{gao2024convergence}. Here, we only want to remark that the proof techniques are similar as for the OU process, unveiled in Appendix~\ref{sec:app_OU}, and do not change in our case as revealed in Proposition~\ref{prop:asymp_same}.

\begin{table}
	\caption{Heuristics for the choice of the time horizon $T$, the step size $h$, and the acceptable score-matching error $\mathcal{E}$ in order for the 2-Wasserstein distance between the generated distribution $\mathcal{L}(\widehat{Z}_{t_k})$ and the true data distribution $p_0$ to be less than or equal to $\varepsilon = o(1)$. Different choices for $f$ and $g$ are considered. The table is split into VE and VP SDEs, and it is assumed that $\norm{X_0}_{L_2} = \mathcal{O}(\sqrt{d})$. }
	\label{tab:examples}
    \centering
    \vspace*{0.5\baselineskip}
    \renewcommand{\arraystretch}{2}
    \begin{tabular}{c|c|c|c|c}
         $f$ & $g$ & $T$ & $h$ & $\mathcal{E}$\\
         \hline
         $0$ & $ae^{bt}$ & $\bigo{\log(\frac{\sqrt{d}}{\varepsilon})}$ & $\bigo{\frac{\eps^3}{d^\frac{3}{2}}}$ & $\bigo{\frac{\eps^2}{\sqrt{d}}}$\\
         $0$ & $(b+at)^c$ & $\bigo{\qty(\frac{d}{\eps^2})^\frac{1}{2c+1}}$ & $\bigo{\frac{\eps^3}{d^\frac{3}{2}}}$ & $\bigo{\frac{\eps^2}{\sqrt{d}}}$\\
         \hline
         $\frac{b}{2}$ & $\sqrt{b}$ & $\bigo{\log\qty(\frac{\sqrt{d}}{\eps})}$ & $\bigo{\frac{\eps}{\sqrt{d} \log \qty(\frac{\sqrt{d}}{\eps})}}$ & $\bigo{\frac{\eps}{\log \qty(\frac{\sqrt{d}}{\eps})}}$\\
         $\frac{b+at}{2}$ & $\sqrt{b+at}$ & $\bigo{\qty(\log\qty(\frac{\sqrt{d}}{\eps}))^\half}$ & $\bigo{\frac{\eps}{\sqrt{d} \log \qty(\frac{\sqrt{d}}{\eps})}}$ & $\bigo{\frac{\eps}{\log \qty(\frac{\sqrt{d}}{\eps})}}$\\
         $\frac{(b+at)^\rho}{2}$ & $(b+at)^{\frac{\rho}{2}}$ & $\bigo{\qty(\log\qty(\frac{\sqrt{d}}{\eps}))^\frac{1}{\rho+1}}$ & $\bigo{\frac{\eps}{\sqrt{d} \log \qty(\frac{\sqrt{d}}{\eps})}}$ & $\bigo{\frac{\eps}{\log \qty(\frac{\sqrt{d}}{\eps})}}$
    \end{tabular}
\end{table}

Next, we compare the rates of our ODE model in Table \ref{tab:examples}
with the analogous results for SDE based models, taken from Table 2 in \cite{gao2025wasserstein}.
We seek the conditions needed to achieve a small sampling error, that is $\mathcal W_2(\mathcal{L}(\widehat{Z}_T), p_0)  \leq  \mathcal O(\eps) = o(1)$. 
Consider first the reverse SDE setting which is analyzed in \citet{gao2025wasserstein}.       
In the VP case, for polynomial $f(t) = (b + at)^\rho/ 2$, one has the requirement (see Corollary 18 and its proof, in particular p.\ 52, in the paper) %
\[
T = K h 
\geq \mathcal  O \left( \log \frac {\sqrt d }\eps \right)^{\frac{1}{\rho + 1}}, \quad h  =  \mathcal O \left ( \frac{\eps^2}{d} \right).
\]
It follows that 
\[
\sqrt d e^{T^{\rho + 1}} h = \sqrt d \,\mathcal O \left ( \frac {\sqrt d} \eps \right)  \cdot \mathcal O \left ( \frac{\eps^2}{d} \right) =  \mathcal O \left(
\eps \right),
\]
so that, in order to achieve $o(1)$ error one needs to take 
\[
h = o\left( \frac{ e^{- T^{\rho +1}}}{\sqrt d } \right).
\]
In particular, in the OU case, corresponding to $\rho = 0$, this implies that one requires $h = o(e^{-T}/ \sqrt d)$, that is an exponentially small in time step size $h$. 

Now consider our reverse ODE setting. 
In the polynomial VP-case, 
$f(t) = (b + at)^\rho/ 2$, Table \ref{tab:examples} shows that we need
\[
T \geq  
\mathcal O \left( \log \frac {\sqrt d} \eps \right)^{\frac{1}{\rho + 1}}, 
\quad 
h = \mathcal O \left ( \frac{\eps}{\sqrt{d} \log\qty(\frac{\sqrt{d}}{\eps})} \right).
\] 
This means that 
\[
\sqrt d T^{\rho + 1 }  h = \mathcal{O}
\left( 
\sqrt d \cdot 
\log \qty(\frac{\sqrt  d} \eps) \cdot 
\frac{\eps}{\sqrt d \log \qty(\frac{\sqrt d}{\eps})} 
\right)
= 
\mathcal O 
\left( 
\eps
\right),
\]
so that, in order to achieve $o(1)$ error, one needs to take 
\[
h = o\left( \frac{1}{\sqrt d \, T^{\rho +1}}\right).
\]
For instance, in the OU case, this means that $h = o(T^{-1}/ \sqrt d)$.

This comparison suggests that, at least in the VP cases under consideration:
\begin{enumerate}
    \item \textit{Why ODE models?} Probability flow models can be more efficient than their SDE counterparts, as they can achieve the same accuracy under much less restrictive step-size requirements---exhibiting polynomial rather than exponential decay in time.
    \item \textit{Curse of dimensionality.}
    As the dimensionality increases, smaller time steps (and hence a larger number of steps) are required, with the dependence scaling on the order of $\sqrt d$. 
\end{enumerate}
\section{Proof of the main result}
\label{sec:proof}

The proof of Theorem~\ref{thm:main_thm} relies on two Propositions that are listed in the following and control the initialization error and the discretization as well as propagated score-matching error, respectively. Their proofs are given in Appendix~\ref{sec:app_proof}.
The first one is a generalization of \citet[Proposition~14]{gao2024convergence}
to our setting. It establishes a control on the initialization error caused by replacing the unknown $\tilde X_0 \sim p_T$ by $Y_0 \sim \hat p_T$ in the reverse flow.  
\begin{proposition}[Initialization error]\label{prop:yt-p0-bound}
Under Assumption \ref{ass:weak_concave_Lip},
    \[
    \mathcal W_2( \mathcal L(Y_T), p_0) \leq 
    C(\alpha_0,  M_0) 
    e^{- \frac 12 \int_0^T g^2(t)|\alpha(t) - M(t)| \de t} \|X_0\|_{L_2}.
    \]
    where $C(\alpha_0, M_0)$ is defined in \eqref{eq:C(alpha, M)}. 
\end{proposition}

\noindent The quantity $C(\alpha_0, M_0)$ measures the increased cost caused by the lack of regularity of $p_0$. If $p_0$ is strongly log-concave, then $C(\alpha_0, M_0) = 1$, as $\tau(\alpha_0, M_0) = 0$. Note that the initialization error will decrease exponentially in $T$ no matter whether $p_0$ is strongly or weakly log-concave. 
Next, we consider the discretization and propagated score-matching error. The following result is a generalization of \citet[Proposition~15]{gao2024convergence}.

\begin{proposition}[Discretization and propagated score matching error]
\label{prop:yt-zt-bound}
    Under Assumptions \ref{ass:weak_concave_Lip}, \ref{ass:Lip_in_time}, and \ref{ass:score_error}, it holds for any $k \in \{1, \dots, K\}$ that
    \begin{align*}
        \norm{Y_{t_k} - \widehat{Z}_{t_k}}_{L_2} &\le \qty(1 - \intk \delta_k(T-t) \de t + \half L_1 h \intk g^2(T-t) \de t)\\
        &\qquad\qquad \cdot e^{\intk f(T-t)\de t} \norm{Y_{t_{k-1}} - \widehat{Z}_{t_{k-1}}}_{L_2}\\
        &\qquad + \half L_1 h \qty(1 + \theta(T) + \omega(T)) \intk e^{\int_t^{t_k} f(T-s) \de s} g^2(T-t) \de t\\
        &\qquad + \half \mathcal{E} \intk e^{\int_t^{t_k} f(T-s) \de s} g^2(T-t) \de t\\
        &\qquad + \half \sqrt{h} \nu_{k, h} \qty(\intk \qty[e^{\int_t^{t_k} f(T-s) \de s} g^2(T-t) L(T-t)]^2 \de t)^\half,
    \end{align*}
    where $\delta_k(T-t)$, $\theta(T)$, $\omega(T)$, and $\nu_{k,h}$ are defined in \eqref{eq:delta_k}, \eqref{eq:theta(T)}, \eqref{eq:omega(T)}, and \eqref{eq:nu_kh}, respectively.
\end{proposition}

Now, we are ready to prove Theorem~\ref{thm:main_thm}.

\begin{proof}[Proof of Theorem~\ref{thm:main_thm}]
    By the triangle inequality for the 2-Wasserstein distance, we have
    \begin{equation}
    \label{eq:zt-p0-triangle}
        \mathcal{W}_2\qty(\law{\widehat{Z}_T}, p_0) \le \mathcal{W}_2\qty(\law{\widehat{Z}_T}, \law{Y_T}) + \mathcal{W}_2\qty(\law{Y_T}, p_0).
    \end{equation}
    To establish a bound for the first term, we will use Proposition~\ref{prop:yt-zt-bound}. To simplify notation, define
    \begin{align*}
        \beta_{k,h} &\coloneqq \half L_1 h \qty(1 + \theta(T) + \omega(T)) \intk e^{\int_t^{t_k} f(T-s) \de s} g^2(T-t) \de t\\
        &\qquad + \half \mathcal{E} \intk e^{\int_t^{t_k} f(T-s) \de s} g^2(T-t) \de t\\
        &\qquad + \half \sqrt{h} \nu_{k, h} \qty(\intk \qty[e^{\int_t^{t_k} f(T-s) \de s} g^2(T-t) L(T-t)]^2 \de t)^\half,
    \end{align*}
    and recall the definition of $\gamma_{k,h}$ from \eqref{eq:gamma_kh}.
    Then, Proposition~\ref{prop:yt-zt-bound} states that for $k \in \{1,\dots,K\}$
    \begin{equation}
    \label{eq:yt-zt_bound_simpl}
        \norm{Y_{t_k} - \widehat{Z}_{t_k}}_{L_2} \le \gamma_{k,h} e^{\intk f(T-t)\de t} \norm{Y_{t_{k-1}} - \widehat{Z}_{t_{k-1}}}_{L_2} + \beta_{k,h}.
    \end{equation}
    If we pick a coupling between $Y_t$ and $\widehat{Z}_t$ such that $Y_0 = \widehat{Z}_0$ a.s., then by recalling that $T=t_K$ and applying \eqref{eq:yt-zt_bound_simpl} recursively, we get
    \begin{align*}
        &\mathcal{W}_2\qty(\law{\widehat{Z}_T}, \law{Y_T})\\
        &\quad\le \norm{\widehat{Z}_{t_K} - Y_{t_K}}_{L_2}\\
        &\quad\le \qty(\prod_{k=1}^K \gamma_{k,h} e^{\intk f(T-t)\de t}) \norm{Y_0 - \widehat{Z}_0}_{L_2} + \sum_{k=1}^K \qty(\prod_{j=k+1}^K \gamma_{j,h} e^{\int_{t_{j-1}}^{t_j} f(T-t)\de t}) \beta_{k,h}\\
        &\quad= \sum_{k=1}^K \qty(\prod_{j=k+1}^K \gamma_{j,h}) e^{\int_{t_k}^{T} f(T-t)\de t} \beta_{k,h}.
    \end{align*}
    Together with Proposition~\ref{prop:yt-p0-bound}, bounding the second term in \eqref{eq:zt-p0-triangle}, it follows that
    \begin{equation*}
        \mathcal{W}_2\qty(\law{\widehat{Z}_T}, p_0) \le C(\alpha_0, M_0) e^{-\int_0^T g^2(t) \abs{\alpha(t) - M(t)} \de t} \norm{X_0}_{L_2} + \sum_{k=1}^K \qty(\prod_{j=k+1}^K \gamma_{j,h}) e^{\int_{t_k}^{T} f(T-t)\de t} \beta_{k,h}.
    \end{equation*}
    The definitions of $E_0$, $E_1$, and $E_2$ complete the proof.
\end{proof} 
\section{Conclusion}
\label{sec:conclusion}

This paper extends convergence theories for score-based generative models to more realistic data distributions and practical ODE solvers, providing concrete guarantees for the efficiency and correctness of the sampling algorithm in practical applications such as image generation.
In particular, our results extend existing 2-Wasserstein convergence bounds for probability flow ODEs to a significantly broader class of distributions (incl.\ Gaussian mixture models) relaxing the strong log-concavity assumption on the data distribution. We provide a very general result that applies to all possible drift and diffusion functions $f$ and $g$. For a number of examples, including both  variance-preserving as well as variance-exploding SDEs, we translate our error bound to concrete heuristics for the choice of the time scale, step size, and acceptable score-matching error that can be used by practitioners implementing SGMs. Remarkably, the asymptotics remain the same as in the strongly log-concave case and, at least in certain setups, outperform those of SDE-based samplers. %

In future work, it would be interesting to see if the assumptions can be even further relaxed and how this would influence the error bound. Moreover, it may be possible to extend the results to the more general case of vector-valued drift functions $f$ and matrix-valued diffusion functions $g$. Another promising line of research concerns reducing the (potentially very large) dimensionality $d$ to the intrinsic dimension of a lower-dimensional manifold on which the data lie. It remains to be seen whether the error bounds presented here can be adapted to this setting. 
\paragraph{Acknowledgements}
F.I., M.T., and J.L.\ are grateful for partial funding by the Deutsche
Forschungsgemeinschaft (DFG, German Research Foundation) under project numbers 520388526 (TRR391), 543964668 (SPP2298), and 502906238. 
\appendix
\section*{Appendix}

The Appendix is structured as follows:
\begin{itemize}
    \item Appendix~\ref{sec:app_weak_conc} provides the proofs of Proposition~\ref{prop:weak_pt}, \ref{prop:regime-shift}, and \ref{prop:weak_Lip} dealing with the propagation in time of Assumption \ref{ass:weak_concave_Lip}. We start by establishing general results on weak concavity that are used in these proofs and also include bounds for the weak concavity constant $K(t) = \alpha(t) - M(t)$ and the Lipschitz constant $L(t)$. Moreover, we provide an example of a (constructed) distribution that is sub-gaussian but not weakly log-concave.
    \item Appendix~\ref{sec:app_OU} treats the specific case of the Ornstein-Uhlenbeck process and provides the derivation of the corresponding error bound given in Theorem~\ref{thm:err_bound_OU}.
    \item Appendix~\ref{sec:app_interpret} deals with the interpretation of our main result (Theorem~\ref{thm:main_thm}). We establish a regime shift result for the contraction rate $\gamma_{k,h}$, derive a bound for $\theta(T)$ that is used in the arguments of Section~\ref{sec:comp_strong}, and provide the proof of Proposition~\ref{prop:asymp_same}, comparing the asymptotics of our error bound with the one in \cite{gao2024convergence}, which imposes a strong log-concavity assumption.
    \item Appendix~\ref{sec:app_proof} provides the proofs of Proposition~\ref{prop:yt-p0-bound} and \ref{prop:yt-zt-bound}, which establish bounds for the different error sources and constitute the key ingredients for the proof of our main result (Theorem~\ref{thm:main_thm}).
\end{itemize} %
\section{Propagation in time of Assumption \ref{ass:weak_concave_Lip}}
\label{sec:app_weak_conc}

We start this section with general properties of weak concavity that will be used in the proof for its propagation in time. The following result relates the weak convexity profile $\kappa_g(r)$ introduced in Definition \ref{def:conv_profile} to the classical definition of strong convexity. In particular, it says that $(\alpha, M)$-weak concavity implies $(\alpha-M)$-strong concavity if $\alpha-M > 0$.

\begin{lemma}
\label{lem:weak_implies_vague}
    Let $g \in C^1(\R^d)$ and $k: [0,\infty) \to \R$. The following two statements are equivalent:
    \begin{enumerate}[(i)]
        \item $\kappa_g(r) \ge k(r)$ for all $r>0$,
        \item $ \dotprod{\nabla g(x) - \nabla g(y)}{x-y} \ge k(\norm{x-y})\norm{x-y}^2$ for all $x, y \in \R^d$.
    \end{enumerate}
    In particular, if $g$ is $(\alpha, M)$-weakly concave, then
    \begin{align*}
        \dotprod{\nabla g(x) - \nabla g(y)}{x-y} &\le -\alpha\norm{x-y}^2 + \norm{x-y}f_M(\norm{x-y})\\
        &\le -(\alpha - M) \norm{x-y}^2.
    \end{align*}
\end{lemma}
\begin{proof}[Proof of Lemma~\ref{lem:weak_implies_vague}]
    We can rewrite $\kappa_g(r) \ge k(r)$ as
    $$  \inf_{x, y \in \R : \norm{x-y} = r} \left\{ \dotprod{\nabla g(x) - \nabla g(y)}{x-y} \right\} \ge k(r)r^2, \quad r>0. $$
    Since the infimum over a set is bounded below by a constant if and only if each element of the set is greater than or equal to this constant, and the inequality holds for all possible values of $r$, the above display is equivalent to
    $$ \dotprod{\nabla g(x) - \nabla g(y)}{x-y} \ge k(\norm{x-y})\norm{x-y}^2 $$
    for all $x, y \in \R^d$.

    The second part of the statement follows from the fact that $\tanh(t) \le t$ for any $t > 0$ and hence $f_M(\norm{x-y}) = 2\sqrt{M} \tanh(\half \sqrt{M} \norm{x-y}) \le M\norm{x-y}$.
\end{proof}

The next result establishes an equivalence between convexity of a function and boundedness of its Hessian. 

\begin{lemma}
\label{lem:vag_hessian}
    Let $g \in C^2(\R^d)$ and $\beta \in \R$. The following two statements are equivalent:
    \begin{enumerate}[(i)]
        \item $\dotprod{\nabla g(x) - \nabla g(y)}{x-y} \ge \beta \norm{x-y}^2$ for all $x, y \in \R^d$,
        \item $\nabla^2 g(x) \succeq \beta I_d$ for all $x \in \R^d$.
    \end{enumerate}
\end{lemma}
\begin{proof}[Proof of Lemma~\ref{lem:vag_hessian}]
    First, assume that (i) holds. Then, for any $v \in \R^d$, we have
    \begin{align*}
        v^T \nabla^2 g(x) v &= \lim_{t\to 0} \frac{\dotprod{\nabla g(x+tv) - \nabla g(x)}{v}}{t}\\
        &= \lim_{t \to 0} \frac{\dotprod{\nabla g(x+tv) - \nabla g(x)}{x+tv-x}}{t^2}\\
        &\ge \lim_{t \to 0} \frac{\beta \norm{tv}^2}{t^2}\\
        &= \beta \norm{v}^2\\
        &= v^t(\beta I_d)v.
    \end{align*}
    On the other hand, assume that (ii) holds, and define
    $$h(t) = \dotprod{\nabla g(x+t(y-x))}{x-y},$$ 
    so that 
    $$h'(t) = (x-y)^T \nabla^2g(x+t(y-x)) (x-y).$$
    By the mean value theorem, it follows that
    $$\dotprod{\nabla g(x) - \nabla g(y)}{x-y} = h(1) - h(0) = h'(\tau)$$
    for some $\tau \in [0,1]$, and hence
    \begin{align*}
        \dotprod{\nabla g(x) - \nabla g(y)}{x-y} &= (x-y)^T \nabla^2g(x+\tau(y-x)) (x-y)\\
        &\ge (x-y)^T (\beta I_d) (x-y)\\
        &= \beta \norm{x-y}^2. \qedhere
    \end{align*}
\end{proof}

An example for weakly log-concave distribution are Gaussian mixture models.

\begin{example} \label{ex:gmm}
    Let $p(x)$ denote the density function of a one-dimensional Gaussian mixture model with three components given by 
    \begin{equation*}
        0.2\cdot \mathcal{N}(-2, 0.8^2) + 0.5\cdot \mathcal{N}(2, 1^2) + 0.3\cdot \mathcal{N}(5, 0.3^2).
    \end{equation*}
    As proved in \citet[Proposition 4.1]{gentiloni2025beyond}, this is an example of a weakly log-concave distribution. An illustration of the density, log-density, score and derivative of the score is given in Figure \ref{fig:gmm}. It clearly shows that the log-density is strongly concave at ``large scales'' with some local fluctuations. Accordingly, the Hessian $\nabla^2 \log p(x)$ is negative for large enough values of $|x|$ and globally bounded from above.
    
    \begin{figure}
        \centering
        \includegraphics[width=0.9\textwidth]{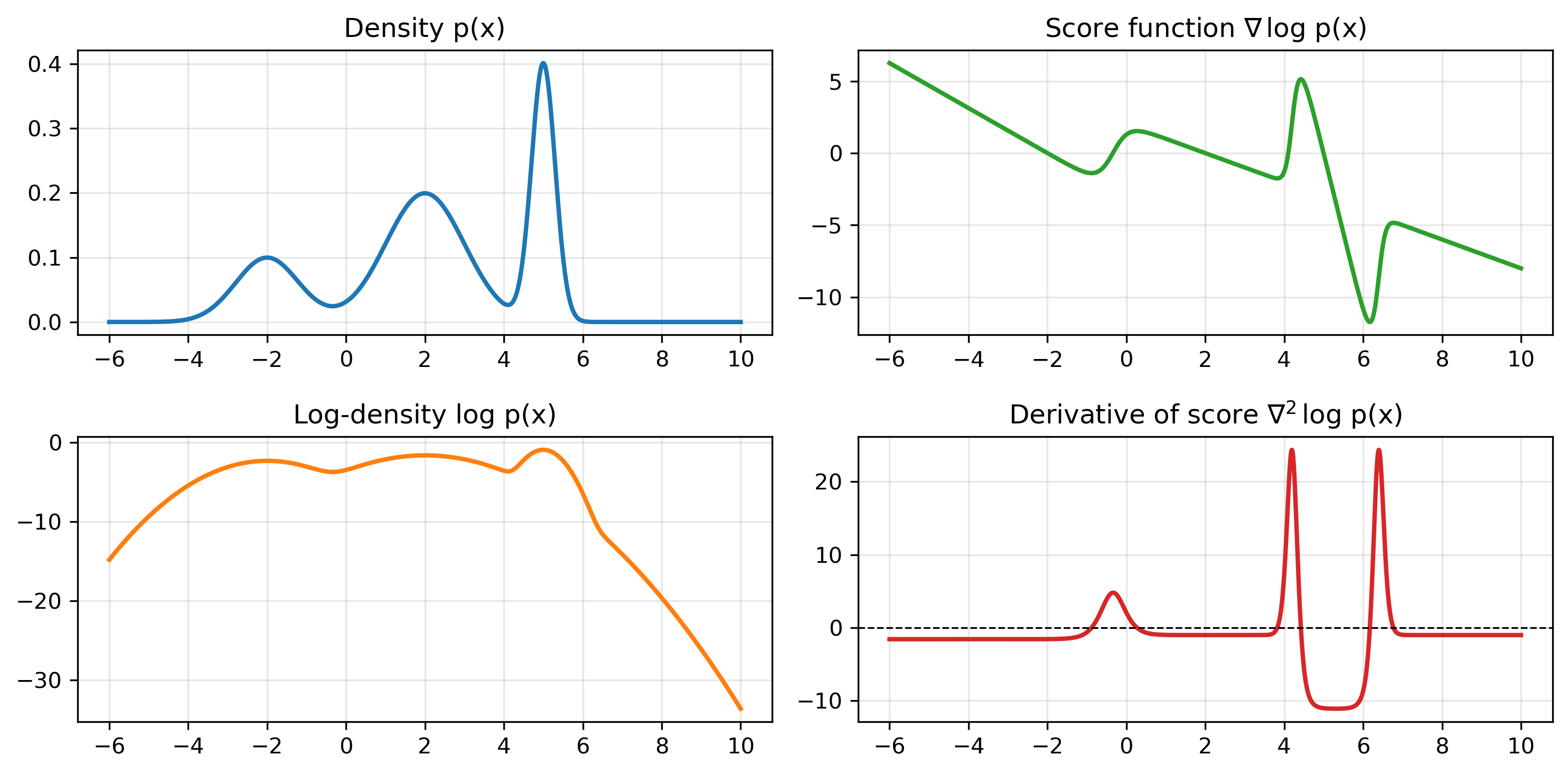}
        \caption{Plots corresponding to a Gaussian mixture model. See Example~\ref{ex:gmm} for more details.}
        \label{fig:gmm}
    \end{figure}
\end{example}

Next, we provide an example of a probability density function that has sub-gaussian tails but does not satisfy the weak log-concavity assumption. Note that it is a very constructed example explicitly meant to reveal the nature of our assumption.

\begin{example} \label{ex:non_weak_concave}
    Consider the probability density function
    \begin{equation*}
        p(x) = \frac{1}{Z} e^{-x^2}\qty(1 + \abs{x}^{\frac{3}{2}}), \quad x \in \R,
    \end{equation*}
    where the normalization constant $Z = \int_{-\infty}^\infty e^{-x^2}(1 + \abs{x}^{\frac{3}{2}}) \de x < \infty$ guarantees its total mass of one. Since, for any $x \in \R$,
    \begin{equation*}
        p(x) \le \frac{1}{Z} e^{-x^2} \qty(1 + e^{\half x^2}) \le \frac{2}{Z} e^{-\half x^2},
    \end{equation*}
    the corresponding distribution is sub-gaussian. However, as
    \begin{equation*}
        \log p(x) = -\log(Z) - x^2 + \log\qty(1+|x|^{\frac{3}{2}})
    \end{equation*}
    and thus
    \begin{equation*}
        \nabla \log p(x) = \begin{cases}
            -2x + \frac{\frac{3}{2} \text{sign}(x) \sqrt{\abs{x}}}{1 + \abs{x}^{3/2}}, &x \ne 0\\
            0, &x = 0 
        \end{cases},
    \end{equation*}
    the score function is infinitely steep at $x=0$. Hence, the Hessian $\nabla^2 \log p(x)$ is unbounded, implying that the distribution cannot be weakly log-concave (cf.\ Lemma~\ref{lem:weak_implies_vague} and \ref{lem:vag_hessian}). An illustration of the involved functions is given in Figure \ref{fig:density_plots}.
    
    \begin{figure}
        \centering
        \includegraphics[width=0.9\textwidth]{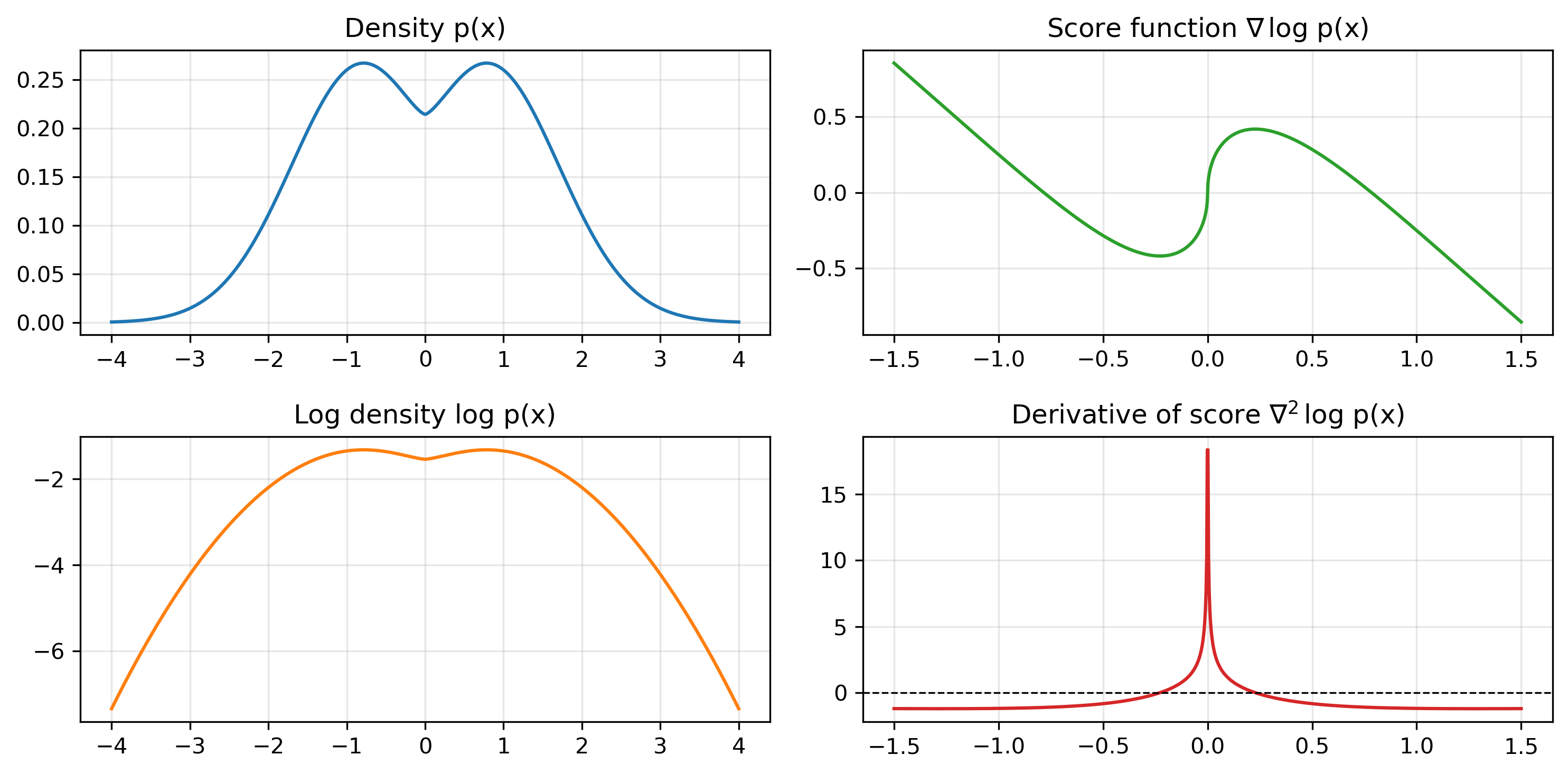}
        \caption{Plots corresponding to a constructed probability density function that is sub-gaussian but not weakly log-concave. See Example~\ref{ex:non_weak_concave} for more details.}
        \label{fig:density_plots}
    \end{figure}
\end{example}

In the following lemma, we list several properties of the convexity profile $\kappa_g(r)$ introduced in Definition \ref{def:conv_profile}. Since the proofs are rather trivial, we do not explicitly state them here.

\begin{lemma}
\label{lem:weak_con_props}
    Let $r>0, \gamma \in \R, c \in \R^d$, and $g, g_1, g_2 \in C^1(\R^d)$. It holds
    \begin{enumerate}[(i)]
        \item \label{it:profile_sum}
        $\kappa_{g_1 + g_2}(r) \ge \kappa_{g_1}(r) + \kappa_{g_2}(r)$,
        \item \label{it:profile_mult_const}
        $\kappa_{\gamma g}(r) = \gamma \kappa_{g}(r)$ for $\gamma > 0$,
        \item \label{it:profile_add_const}
        $\kappa_{g+c}(r) = \kappa_{g}(r)$,
        \item \label{it:profile_mult_input}
        $\kappa_{g(\gamma x)}(r) = \gamma^2 \kappa_g(\abs{\gamma} r)$,
        \item \label{it:profile_norm}
        $\kappa_{\gamma \norm{\cdot}^2}(r) = 2\gamma$.
    \end{enumerate}
\end{lemma}

As we will see in the proof of Proposition~\ref{prop:weak_pt}, the density $p_t$ can be written as a convolution of $p_0$ with a Gaussian distribution. We are interested in how the weak log-concavity of $p_0$ is carried over to $p_t$ by this transformation. The following theorem provides an important result in this context. %
It was originally published in \cite[Theorem~2.1]{conforti2024weak} and restated in \cite[Theorem~B.3]{gentiloni2025beyond}. 

\begin{theorem}
\label{thm:weak_convol}
    Fix $M > 0$ and define
    $$ \mathcal{F}_M = \{g \in C^1(\R^d): \kappa_g(r) \ge -r^{-1} f_M(r)\}. $$ 
    Then for all $0 < v < \infty$, it holds that
    $$ -\log g \in \mathcal{F}_M \Rightarrow -\log(S_v g) \in \mathcal{F}_M, $$
    where $(S_v)_{v \ge 0}$ denotes the semigroup generated by a standard Brownian motion on $\R^d$, defined as
    $$S_v g(x) = \int (2\pi v)^{-d/2} \exp \qty(-\frac{\norm{x-y}^2}{2v}) g(y) dy.$$
\end{theorem}

The connection between $\mathcal{F}_M$ and weak convexity is revealed in the following lemma.

\begin{lemma}
\label{lem:weak_FM}
    If $h$ is $(\alpha, M)$-weakly convex, then $h - \half \alpha \norm{\cdot}^2 \in \mathcal{F}_M$. 
\end{lemma}
\begin{proof}[Proof of Lemma~\ref{lem:weak_FM}]
    By Lemma~\ref{lem:weak_con_props}\eqref{it:profile_sum} and \eqref{it:profile_norm} together with the weak convexity of $h$, we have
    \begin{equation*}
        \kappa_{h - \half \alpha \norm{\cdot}^2}(r) \ge \kappa_h(r) + \kappa_{-\half \alpha \norm{\cdot}^2} \ge  \alpha - r^{-1}f_M(r) + 2 \qty(-\half \alpha) = -r^{-1} f_M(r). \qedhere
    \end{equation*}
\end{proof}

\subsection{Propagation in time of weak log-concavity}

Now, we are ready to present the proof of Proposition~\ref{prop:weak_pt}, establishing the weak log-concavity of $p_t$ given that $p_0$ is $(\alpha_0, M_0)$-weakly log-concave.

\begin{proof}[Proof of Proposition~\ref{prop:weak_pt}]
    Observe that $p_t(x) = \int p_{t|0}(x|y) p_0(y) dy$, where $p_{t|0}(\cdot|y)$ denotes the conditional density of $X_t$ given $X_0 = y$. From equation \eqref{eq:forw_SDE_solution}, it follows that
    \begin{equation*}
        p_{t|0}(x|y) = (2\pi c_1(t))^{-d/2} \exp \qty(- \frac{\norm{x - c_0^{-1}(t)y}^2}{2c_1(t)})
    \end{equation*}
    with
    \begin{equation}\label{eq:def_c0_c1}
        c_0(t) = e^{\int_0^{t} f(s) ds}, \quad
        c_1(t) = \int_0^{t} e^{-2\int_s^{t} f(v)dv} g^2(s) ds,
    \end{equation}
    which yields
    \begin{equation}
    \label{eq:pt_conv}
        p_{t}(x) = \int (2\pi c_1(t))^{-d/2} \exp \qty(- \frac{\norm{x - c_0^{-1}(t)y}^2}{2c_1(t)}) p_0(y) dy.
    \end{equation} 
    We can write the argument of the exponential function within $p_{t}$ as
    \begin{align*}
        - \frac{\norm{x - c_0^{-1}(t)y}^2}{2c_1(t)}
        &= - \frac{\norm{x - c_0^{-1}(t)y}^2}{2c_1(t)} - \half \alpha_0 \norm{y}^2 + \half \alpha_0 \norm{y}^2\\
        &= - \frac{\norm{x - c_0^{-1}(t)y}^2 + \alpha_0 c_1(t)\norm{y}^2}{2c_1(t)} + \half \alpha_0 \norm{y}^2\\
        &= - \frac{\norm{x}^2 - 2c_0^{-1}(t)\dotprod{x}{y} + c_0^{-2}(t) \norm{y}^2 + \alpha_0 c_1(t) \norm{y}^2}{2c_1(t)} + \half \alpha_0 \norm{y}^2.
    \end{align*}
    Defining $c_\alpha(t) = c_0^{-2}(t) + \alpha c_1(t)$ and completing the square further yields
    \begin{align*}
        - \frac{\norm{x - c_0^{-1}(t)y}^2}{2c_1(t)}
        &= - \frac{\norm{x}^2 - 2c_0^{-1}(t)\dotprod{x}{y} + c_{\alpha_0}(t) \norm{y}^2}{2c_1(t)} + \half \alpha_0 \norm{y}^2\\
        &= - \frac{c_{\alpha_0}^{-1}(t) \norm{x}^2 - 2(c_{\alpha_0}(t) c_0(t))^{-1}\dotprod{x}{y} + \norm{y}^2}{2 c_{\alpha_0}^{-1}(t) c_1(t)} + \half \alpha_0 \norm{y}^2\\
        &= - \frac{\norm{(c_{\alpha_0}(t) c_0(t))^{-1} x - y}^2}{2 c_{\alpha_0}^{-1}(t) c_1(t)} + \frac{(c_{\alpha_0}(t) c_0(t))^{-2} - c_{\alpha_0}^{-1}(t)}{2 c_{\alpha_0}^{-1}(t) c_1(t)} \norm{x}^2 + \half \alpha_0 \norm{y}^2\\
        &= - \frac{\norm{c_4(\alpha_0, t) x - y}^2}{2 c_3(\alpha_0, t)} - \half c_2(\alpha_0, t) \norm{x}^2 + \half \alpha_0 \norm{y}^2,
    \end{align*}
    where 
    \begin{align*} %
        c_2(\alpha, t) &\coloneqq - \frac{(c_\alpha(t) c_0(t))^{-2} - c_\alpha^{-1}(t)}{c_\alpha^{-1}(t) c_1(t)},\\
        c_3(\alpha, t) &\coloneqq c_\alpha^{-1}(t) c_1(t), \notag \\
        c_4(\alpha, t) &\coloneqq (c_\alpha(t) c_0(t))^{-1}.\notag 
    \end{align*}
    Altogether, we get 
    \begin{align*}
        p_{t}(x) 
        &= \int (2\pi c_1(t))^{-d/2} \exp \qty(- \frac{\norm{c_4(\alpha_0,t) x - y}^2}{2 c_3(\alpha_0,t)} - \half c_2(\alpha_0,t) \norm{x}^2 + \half \alpha_0 \norm{y}^2) p_0(y) dy\\
        &= \exp \qty(-\half c_2(\alpha_0,t) \norm{x}^2) \qty(\frac{c_1(t)}{c_3(t)})^{-d/2}\\
        &\qquad \cdot\int (2\pi c_3(t))^{-d/2} \exp \qty(- \frac{\norm{c_4(\alpha_0,t) x - y}^2}{2 c_3(\alpha_0,t)} + \half \alpha_0 \norm{y}^2) p_0(y) dy\\
        &= \exp \qty(-\half c_2(\alpha_0,t) \norm{x}^2) \qty(\frac{c_1(t)}{c_3(t)})^{-d/2} 
        S_{c_3(\alpha_0,t)} \qty(\exp\qty(\half \alpha_0 \norm{\cdot}^2) p_0) \qty(c_4(\alpha_0,t) x),
    \end{align*}
    or equivalently
    \begin{equation*} %
        - \log p_{t} (x) = \half c_2(\alpha_0,t) \norm{x}^2 + \frac{d}{2} \log \qty(\frac{c_1(t)}{c_3(t)}) - \log S_{c_3(\alpha_0,t)} \qty(\exp\qty(\half \alpha_0 \norm{\cdot}^2) p_0) \qty(c_4(\alpha_0,t) x).
    \end{equation*}
    By Lemma~\ref{lem:weak_con_props}\eqref{it:profile_sum} and \eqref{it:profile_mult_input}, this implies that
    \begin{equation*}
        \kappa_{-\log p_{t}}(r) \ge \kappa_{\half c_2(\alpha_0, \tilde t) \norm{\cdot}^2}(r) + c_4^2(\alpha_0, t) \kappa_{- \log S_{c_3(\alpha_0,t)} \qty(\exp\qty(\half \alpha_0 \norm{\cdot}^2) p_0)} \big(c_4(\alpha_0, t) r\big).
    \end{equation*}
    Since $p_0$ is assumed to be $(\alpha_0, M_0)$-weakly log-concave, it follows by Lemma~\ref{lem:weak_FM} that
    $$ -\log \qty(\exp\qty(\half \alpha_0 \norm{\cdot}^2) p_0) = -\log p_0 + \half \alpha_0 \norm{\cdot}^2 \in \mathcal{F}_{M_0} $$
    and thus, by Theorem~\ref{thm:weak_convol}, that
    $$ - \log S_{c_3(\alpha_0,t)} \qty(\exp\qty(\half \alpha_0 \norm{\cdot}^2) p_0) \in \mathcal{F}_{M_0}. $$
    This result together with Lemma~\ref{lem:weak_con_props}\eqref{it:profile_norm} further yields
    \begin{align*}
        \kappa_{-\log p_{t}}(r) &\ge c_2(\alpha_0, t) - c_4^2(\alpha_0, t) \big(c_4(\alpha_0, t) r\big)^{-1} f_{M_0} \big(c_4(\alpha_0, t) r\big)\\
        &= c_2(\alpha_0, t) - r^{-1} f_{M_0 c_4^2(\alpha_0, t)} (r),
    \end{align*}
    where in the last equality we used the fact that by definition $cf_M(cr) = f_{c^2M}(r)$ for any $c, M, r > 0$.

    \begingroup
    \allowdisplaybreaks
    The following simple but tedious calculations finally show that $\alpha(t) = c_2(\alpha_0, t)$ and $M(t) = M_0 c_4^2(\alpha_0, t)$, completing the proof. In particular, we have
    \begin{align*}
        c_2(\alpha, t) &= - \frac{(c_\alpha(t) c_0(t))^{-2} - c_\alpha^{-1}(t)}{c_\alpha^{-1}(t) c_1(t)}\\
        &= \frac{1 - c_\alpha^{-1}(t) c_0^{-2}(t)}{c_1(t)}\\
        &= \frac{1}{c_1(t)} \qty(1 - \frac{1}{c_\alpha(t) c_0^2(t)})\\
        &= \frac{1}{c_1(t)} \qty(1 - \frac{1}{\qty(c_0^{-2}(t) + \alpha c_1(t)) c_0^2(t)})\\
        &= \frac{1}{c_1(t)} \qty(1 - \frac{1}{1 + \alpha c_0^2(t) c_1(t)})\\
        &= \frac{\alpha c_0^2(t)}{1 + \alpha c_0^2(t) c_1(t)}\\
        &= \frac{1}{\alpha^{-1} c_0^{-2}(t) + c_1(t)},
    \end{align*}
    and
    \begin{equation*}
        c_4(\alpha, t) = \frac{1}{c_\alpha(t) c_0(t)}
        = \frac{1}{\big(c_0^{-2}(t) + \alpha c_1(t)\big) c_0(t)}
        = \frac{c_0(t)}{1 + \alpha c_0^2(t) c_1(t) }. \qedhere
    \end{equation*}
    \endgroup
\end{proof}

\begin{remark}
    It can be easily checked that $c_0(t)$, $c_\alpha(t)$, $c_2(\alpha, t)$, and $c_4(\alpha, t)$ are positive for any $\alpha > 0$ and $t \ge 0$. Moreover, $c_1(t)$ and $c_3(\alpha, t)$ are strictly positive for any  $\alpha > 0$ and $t > 0$ and zero for $t=0$.
\end{remark}

Next, we prove the regime shifting result, Proposition~\ref{prop:regime-shift}, dealing with the switch of $p_t$ from being weakly to strongly log-concave around $t = \tau(\alpha_0, M_0)$.

\begin{proof}[Proof of Proposition~\ref{prop:regime-shift}]
    If $\alpha_0 - M_0>0$, the result trivially holds with $\tau(\alpha_0, M_0) = 0$, due to the log-concavity preservation result in  \cite[Proposition~7]{gao2024convergence}. 
    So we only need to consider the case that $\alpha_0 - M_0\leq 0$.
    By Lemma~\ref{lem:weak_implies_vague}, $p_t$ is $(\alpha(t) - M(t))$-strongly log-concave if $\alpha(t) - M(t) > 0$ which holds if and only if
    \begin{align*}
        \frac{\alpha_0 c_0^2(t)}{1 + \alpha_0 c_0^2(t) c_1(t)} - \frac{M_0 c_0^2(t)}{\big(1 + \alpha_0 c_0^2(t) c_1(t)\big)^2} &> 0 \\
        \Leftrightarrow\quad \alpha_0 c_0^2(t) \Big(1 + \alpha_0 c_0^2(t) c_1(t)\Big) - M_0 c_0^2(t) &> 0\\
        \Leftrightarrow\quad \alpha_0 \Big(1 + \alpha_0 c_0^2(t) c_1(t)\Big) - M_0 &> 0\\
        \Leftrightarrow\quad \alpha_0 + \alpha_0^2 c_0^2(t) c_1(t) &> M_0\\
        \Leftrightarrow\quad c_0^2(t) c_1(t) &> \frac{M_0 - \alpha_0}{\alpha_0^2}. \numberthis \label{eq:tau-ineq}
    \end{align*}
    By recalling the definition \eqref{eq:def_c0_c1} of $c_0$ and $c_1$, we have that
    \begin{equation}\label{eq:c0-c1-rel}
        c_1(t) = \int_0^{t} e^{-2\int_s^{t} f(v)dv} g^2(s) \de s = 
        e^{-2\int_0^{t} f(s) \de s}
        \int_0^{t} e^{2\int_0^{s} f(v) \de v} g^2(s) \de s = 
        \frac{1}{c_0^2(t)}   \int_0^{t} e^{2\int_0^{s} f(v) \de v} g^2(s) \de s.
    \end{equation}
    Hence, condition \eqref{eq:tau-ineq} can be rewritten as 
    \begin{equation*}%
        \int_0^{t} e^{2\int_0^{s} f(v) \de v} g^2(s) \de s > \frac{M_0 - \alpha_0}{\alpha_0^2}. \qedhere
    \end{equation*}
\end{proof}

The following lemma provides a lower bound for the weak concavity constant $K(t) = \alpha(t) - M(t)$. It is used at several occasions within the paper: when comparing our error bound to the strongly log-concave case in Section~\ref{sec:comp_strong}, to establish the more accessible error bound for the OU process in Theorem~\ref{thm:err_bound_OU}, and in the proof of Proposition~\ref{prop:yt-p0-bound} bounding the initialization error.

\begin{lemma}\label{lem:kt_inf}
    Let $K(t) = \alpha(t) - M(t), \, t\geq 0$. Then the following holds: 
    \begin{equation}\label{eq:kt_bound}
        K(t)     \geq - |\alpha_0 - M_0|  \, 
        \min 
        \left\{ 
        e^{ 2 \int_0^{t} f(s) ds}
        ,
        \frac{ e^{ 2 \int_0^{t} f(s) ds}}{(\alpha_0 \int_0^{t} e^{2\int_0^{s} f(v) \de v} g^2(s) \de s)^2}
        \right\} \, , 
        \qquad t \geq 0. 
    \end{equation}
    In particular, for any finite time $T>0$, it holds that 
    \begin{equation}\label{eq:kt_inf}
       \inf_{0 \leq t \leq T} K(t) \geq - \frac{|\alpha_0 - M_0|}{\alpha_0^2 \wedge 1}
       \xi(T),
    \end{equation}
    where $\xi(T)$ is defined in \eqref{eq:xi(T)}.

\end{lemma}

For example, in the OU case, for small $t$,  \eqref{eq:kt_bound} would read $K(t) \geq - |\alpha_0 - M_0| e^{2t}$. This is very tight around $t=0$, where the bound is close to the exact value $K(0) = \alpha_0 - M_0$. 
In the VP case, for large $t$, \eqref{eq:kt_bound}  reads \[
K(t) \geq - \frac{|\alpha_0 - M_0|}{\alpha_0^2} \frac{e^{\mathcal B(t)}}{(e^{\mathcal B(t)} -1)^2}, \]
which is close to zero for large $t$. This is enough for our purpose, as, intuitively, our results only require a control of $K(t) = \alpha(t) - M(t)$ when it is negative, that is when $p_t$ deviates from strong log-concavity. But, thanks to the regime shifting result in Proposition~\ref{prop:regime-shift}, we know this can only happen up to a finite time $\tau(\alpha_0, M_0)$. See also Example~\ref{ex:kt_vp} below for more details on the VP case. 

\begin{proof}[Proof of Lemma~\ref{lem:kt_inf}]
    If $\alpha_0 - M_0 > 0$, it holds that $K(t) > 0$ as a consequence of log-concavity preservation \citep[Proposition~7]{gao2024convergence} and then \eqref{eq:kt_bound} is trivially satisfied. So we only need to consider the case that $\alpha_0 - M_0 < 0$. For any $t \ge 0$, by means of simple algebra, we can write 
    \begin{align*}
        K(t) &= \frac{\alpha_0 c^2_0(t)}{1 + \alpha_0 c^2_0(t) c_1(t)} - \frac{M_0 c^2_0(t)}{\big(1 + \alpha_0 \, c^2_0(t) c_1(t)\big)^2}
        \\
        & \geq 
        \frac{\alpha_0 c^2_0(t)}{\left( 1 + \alpha_0 c^2_0(t) c_1(t)\right )^2} - \frac{M_0 c^2_0(t)}{\big(1 + \alpha_0 \, c^2_0(t) c_1(t)\big)^2}
        \\
        &= 
        - \frac{ |\alpha_0 - M_0| c_0^2(t)}{\left (1 + \alpha_0 c^2_0(t) c_1(t) \right)^2 } \numberthis \label{eq:kt_step}
        \\
        & \geq - |\alpha_0 - M_0| c_0^2(t). 
    \end{align*}
    Alternatively, starting from \eqref{eq:kt_step}, we have
     \begin{equation*}
          K(t) \geq - \frac{|\alpha_0 - M_0|}{\alpha_0^2} \frac{c_0^2(t)}{(c_0^2(t) c_1(t))^2}.
    \end{equation*}
    Finally, by combining the inequalities above we conclude:
    \begin{align*}
        K(t)     &\geq 
     - |\alpha_0 - M_0|
     \min
     \left\{ 
     c_0^2(t),
      \frac{1}{(\alpha_0 c_0(t) c_1(t))^2}
     \right\} \, \numberthis \label{eq:kt_combi1}
     \\ & \geq 
     - \frac{|\alpha_0 - M_0|}{\alpha_0^2 \wedge 1}
     \min
     \left\{ 
     c_0^2(t),
      \frac{1}{(c_0(t) c_1(t))^2}
     \right\}. \numberthis \label{eq:kt_combi2}
    \end{align*}
    Equation \eqref{eq:kt_combi1} can be rewritten as \eqref{eq:kt_bound} by recalling the definitions of $c_0$ and $c_1$ given in \eqref{eq:def_c0_c1}.
    By taking infima over $t \ge 0$ in \eqref{eq:kt_combi2}, we get \eqref{eq:kt_inf}.
\end{proof}

\begin{example}\label{ex:kt_vp}
    We derive explicit expressions for the regime-shift time 
    and the weak-concavity constant in the VP case, i.e.\ for $f(t) = \beta(t)/2$ and $g(t) = \sqrt {\beta(t)}$. 
    
    Let $\mathcal B(t) = \int_0^t \beta(s) \de s$.
    Then, from the definition \eqref{eq:tau_def} of $\tau(\alpha_0, M_0)$,  we get
    \begin{align*}
        \int_0^{t} e^{2\int_0^{s} f(v) \de v} g^2(s) \de s &> \frac{M_0 - \alpha_0}{\alpha_0^2} \\
        \Leftrightarrow\quad e^{\mathcal B(t) } -1 &> \frac{M_0 - \alpha_0}{\alpha^2_0},
    \end{align*}
    and consequently
    \begin{align*}         
    \tau(\alpha_0, M_0) &= \mathcal B^{-1}\left( \log \left(\frac{\alpha_0^2 + M_0 - \alpha_0}{\alpha_0^2} \right)\right)\,,
     \end{align*}
     where the inverse function $\mathcal B^{-1}(\cdot)$ is well-defined as $\mathcal B(\cdot)$ is continuous and strictly increasing. 
     In particular, for the Ornstein-Uhlenbeck process, i.e.\ $f(t)=1$ and $g(t) = \sqrt 2$, we have 
    \begin{align*}
        \tau(\alpha_0, M_0) = \log \sqrt{\frac{\alpha_0^2 + M_0 - \alpha_0}{\alpha_0^2}}. 
    \end{align*}

    Next, we turn to the weak concavity constant $K(t)$.
    By recalling the definition \eqref{eq:def_c0_c1} of 
    $c_0(t)$, $c_1(t)$, and by relation \eqref{eq:c0-c1-rel}, we have
    \[
    c_0^2(t) = e^{\mathcal B(t)}\, , \qquad c_0^2(t) c_1(t) = e^{\mathcal B(t)} - 1.
    \]
    Hence, from the definitions \eqref{eq:alpha_def}, \eqref{eq:M_def} 
    of $\alpha(t)$, $M(t)$ we get
\begin{equation}\label{eq:kt-vp}
    K(t) = \alpha(t) - M(t) = \frac{\alpha_0 e^{\mathcal B(t)}}{1 +\alpha_0( e^{\mathcal B(t)} -1) } -
 \frac{M_0 e^{\mathcal B(t)}}{\left(1 +\alpha_0( e^{\mathcal B(t)} -1) \right)^2 } \,, \quad t \geq 0,
\end{equation}
for positive $\alpha_0, M_0$. We remark that, as $t \to \infty$,  if $\mathcal B (t) \to \infty$,  one has $K(t) \to 1$, in agreement with the limiting standard Gaussian behavior of the forward diffusion process. If, in addition, $\alpha_0 =1 $, then $K(t)$ is guaranteed to be strictly increasing, since  $\mathcal B(t)$ is strictly increasing.  
See Figure \ref{fig:alphaM} for a graphical representation of possible behaviors.
\end{example}

\begin{figure}
    \centering
    \includegraphics[width=0.5\linewidth]{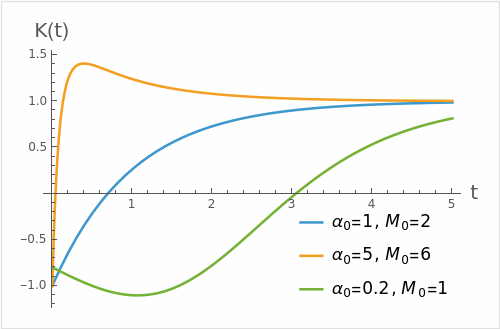}
    \caption{Plot of $K(t) = \alpha(t) - M(t), \, t\geq 0$ for different values of $\alpha_0, M_0$, in the OU case.}
    \label{fig:alphaM}
\end{figure}

\subsection{Propagation in time of Lipschitz continuity}

Next, we present the proof of Proposition~\ref{prop:weak_Lip} which establishes the Lipschitz smoothness of $\log p_t$ given that Assumption \ref{ass:weak_concave_Lip} holds, i.e.\ assuming that $\log p_0$ is $(\alpha_0, M_0)$-weakly concave and $L_0$-smooth.

\begin{proof}[Proof of Proposition~\ref{prop:weak_Lip}]
    We use similar arguments as in the proof of \cite[Lemma~9]{gao2025wasserstein}. With a change of variable, we can rewrite \eqref{eq:pt_conv} as
    \begin{equation*}
        p_{t}(x) = \qty(c_0(t))^d \int (2\pi c_1(t))^{-d/2} \exp \qty(- \frac{\norm{x - y}^2}{2c_1(t)}) p_0(c_0(t) y) dy
    \end{equation*}
    with $c_0(t)$ and $c_1(t)$ defined in \eqref{eq:def_c0_c1}. Letting
    \begin{equation*}
        q^t_0(x) \coloneqq p_0(c_0(t)x), \quad
        q^t_1(x) \coloneqq (2\pi c_1(t))^{-d/2} \exp \qty(- \frac{\norm{x}^2}{2c_1(t)}),
    \end{equation*}
    and $q_0^t \ast q_1^t$ denote their convolution, this implies that
    \begin{equation*}
        \nabla^2 \log p_t(x) = \nabla^2 \log \qty(q_0^t \ast q_1^t)(x).
    \end{equation*}
    We further define $\varphi_k^t = -\log q_k^t$ for $k \in \{0,1\}$. An intermediate result of \citet[Proposition~7.1]{saumard2014log}, that does not make use of the strong log-concavity assumption, yields
    \begin{align*}
        \nabla^2 (-\log p_{t})(z) &= -\Var{\nabla \varphi_0^t(X) | X+Y=z} + \E{\nabla^2 \varphi_0^t(X) | X+Y = z}\\
        &= -\Var{\nabla \varphi_1^t(Y) | X+Y=z} + \E{\nabla^2 \varphi_1^t(Y) | X+Y = z}.
    \end{align*}
    Let $v \in \R^d$. By Cauchy-Schwartz inequality and the $L_0$-Lipschitz continuity of $\nabla \log p_0$, we have
    \begin{align*}
        v^t \nabla^2 \varphi_0^t(x) v &= \lim_{t \to 0} \frac{\dotprod{\nabla \varphi_0^t(x+tv) - \nabla \varphi_0^t(x)}{v}}{t}\\
        &\le \lim_{t \to 0} \frac{\norm{\nabla \varphi_0^t(x+tv) - \nabla \varphi_0^t(x)} \cdot \norm{v}}{t}\\
        &= \lim_{t \to 0} \frac{\abs{c_0(t)} \cdot \norm{\nabla \log p_0^t(c_0(t)(x+tv)) - \nabla \log p_0^t(c_0(t)x)} \cdot \norm{v}}{t}\\
        &\le \lim_{t\to 0} \frac{\abs{c_0(t)} \cdot L_0 \norm{c_0(t)tv}\cdot \norm{v}}{t}\\
        &= c_0^2(t) L_0 \norm{v}^2\\
        &= v^T c_0^2(t) L_0 I_d v.
    \end{align*}
    Hence, for all $x \in \R^d$,
    \begin{equation*}
        \nabla^2 \varphi_0^t(x) \preceq c_0^2(t) L_0 I_d.
    \end{equation*}
    Moreover, recall that
    \begin{equation*}
        \varphi_1^t(x) = - \log q_1^t(x) = \frac{d}{2} \log\qty(2\pi c_1(t)) + \frac{\norm{x}^2}{2c_1(t)}
    \end{equation*}
    and thus
    \begin{equation*}
        \nabla^2 \varphi_1^t(x) = \frac{1}{c_1(t)} I_d
    \end{equation*}
    for all $x \in \R^d$.
    Since covariance matrices are always positive semi-definite, this finally leads to 
    \begin{equation*}
        \nabla^2 (-\log p_t)(z) \preceq \min\left\{\frac{1}{c_1(t)}, c_0^2(t) L_0\right\} I_d.
    \end{equation*}
    
    Note that from $A \preceq B$, we cannot directly follow that $\norm{A} \le \norm{B}$. However, if $C \preceq A \preceq B$, then we have $\norm{A} \le \max\{\norm{B}, \norm{C}\}$. In particular, $0 \preceq A \preceq B$ implies $\norm{A} \le \norm{B}$. This result can be easily proven using the fact that the (spectral) norm of a symmetric matrix is given by its largest absolute eigenvalue. 
    
    From Proposition~\ref{prop:weak_pt} together with Lemma~\ref{lem:vag_hessian}, we get
    \begin{equation*}
        \nabla^2 (-\log p_{t})(z) \succeq (\alpha(t) - M(t)) I_d.
    \end{equation*}
    In case $\alpha(t) - M(t) > 0$, this yields
    \begin{align*}
        \norm{\nabla^2 \log p_{t}(x)} &\le \min\left\{ \frac{1}{c_1(t)}, c_0^2(t) L_0 \right\}\\
        &= \max \left\{ \min\left\{ \frac{1}{c_1(t)}, c_0^2(t) L_0 \right\}, -\qty(\alpha(t) - M(t)) \right\} = L(t).
    \end{align*}
    If, on the other hand, $\alpha(t) - M(t) < 0$, it follows that
    \begin{align*}
        \norm{\nabla^2 \log p_{t}(x)} &\le \max \left\{ \min\left\{ \frac{1}{c_1(t)}, c_0^2(t) L_0 \right\}, \abs{\alpha(t) - M(t)} \right\}\\
        &= \max \left\{ \min\left\{ \frac{1}{c_1(t)}, c_0^2(t) L_0 \right\}, -\qty(\alpha(t) - M(t)) \right\} = L(t). \qedhere
    \end{align*}
\end{proof}

The following lemma provides an upper bound for the Lipschitz constant $L(t)$. It is used when comparing our error bound to the strongly log-concave case in Section~\ref{sec:comp_strong} and to establish the more accessible error bound for the OU process in the proof of Theorem~\ref{thm:err_bound_OU}.

\begin{lemma}[Upper bound for $L(t)$]
    \label{lem:Lt_bound}
    It holds that
    \begin{equation*}
        \sup_{0 \le t \le T} L(t) \le \max\left\{L_0 \vee 1, \frac{\abs{\alpha_0 - M_0}}{\alpha_0^2 \wedge 1}\right\} \eta(T),
    \end{equation*}
    where 
    \begin{equation} \label{eq:eta}
        \eta(T) \coloneqq \sup_{0 \le t \le T} \min \left\{ e^{2 \int_0^{t} f(s) ds}, \frac{1}{\int_0^{t} e^{-2\int_s^{t} f(v)dv} g^2(s) ds} \right\}.
    \end{equation}
    Moreover, for $\xi(T)$ defined in Lemma~\ref{lem:kt_inf}, we have $\xi(T) \le \eta(T)$.
\end{lemma}
\begin{proof}[Proof of Lemma~\ref{lem:Lt_bound}]
    Using the definition of $c_0(t)$ and $c_1(t)$ in \eqref{eq:def_c0_c1}, Proposition~\ref{prop:weak_Lip} states that
    \begin{align*}
        L(t) = \max \left\{ \min\left\{ \frac{1}{c_1(t)}, c_0^2(t) L_0\right\}, -(\alpha(t)-M(t))\right\}.
    \end{align*}
    By Lemma~\ref{lem:kt_inf}, it holds that
    \begin{align*}
        \sup_{0 \le t \le T} -(\alpha(t)-M(t)) \le \frac{\abs{\alpha_0 - M_0}}{\alpha_0^2 \wedge 1} \xi(T)
    \end{align*}
    with (see \eqref{eq:kt_combi2})
    \begin{equation*}
        \xi(T) = \sup_{0 \le t \le T} \min \left\{ c_0^2(t), \frac{1}{(c_0(t) c_1(t))^2} \right\}.
    \end{equation*}
    Furthermore, we have
    \begin{align*}
         \sup_{0 \le t \le T} \min\left\{ \frac{1}{c_1(t)}, c_0^2(t) L_0\right\} &\le (L_0 \vee 1) \sup_{0 \le t \le T} \min\left\{ \frac{1}{c_1(t)}, c_0^2(t) \right\}\\
         &= (L_0 \vee 1) \eta(T).
    \end{align*}
    The result follows if we can show that $\xi(T) \le \eta(T)$.
    For that, consider $t \ge 0$ for which
    \begin{equation*}
        \frac{1}{c_1(t)} \le c_0^2(t).
    \end{equation*}
    For those values of $t$, it also holds that
    \begin{equation*}
        \frac{1}{(c_0(t) c_1(t))^2} = \frac{1}{c_0^2(t)} \qty(\frac{1}{c_1(t)})^2 \le \frac{1}{c_0^2(t)} \qty(c_0^2(t))^2 = c_0^2(t)
    \end{equation*}
    and
    \begin{equation*}
        \frac{1}{(c_0(t) c_1(t))^2} = \frac{1}{c_0^2(t)} \frac{1}{c_1(t)} \frac{1}{c_1(t)} \le \frac{1}{c_0^2(t)} \frac{1}{c_1(t)} c_0^2(t) = \frac{1}{c_1(t)}.
    \end{equation*}
    It follows that
    \begin{equation*}
        \min \left\{ c_0^2(t), \frac{1}{(c_0(t) c_1(t))^2} \right\} \le \min\left\{ \frac{1}{c_1(t)}, c_0^2(t) \right\}
    \end{equation*}
    and consequently $\xi(T) \le \eta(T)$.
\end{proof}

\section{Error bound for the Ornstein-Uhlenbeck process}
\label{sec:app_OU}

In this section, we derive the explicit error bound given in Theorem~\ref{thm:err_bound_OU} for the specific case of $f(t) \equiv 1$ and $g(t) \equiv \sqrt{2}$, resulting in the OU process.
Many quantities simplify in this case. In particular, the bounds for the different error types in Theorem~\ref{thm:main_thm} read
\begin{align}
    E_0(f,g,T) &= C(\alpha_0, M_0) e^{- \int_0^T \abs{\alpha(t) - M(t)} \de t} \norm{X_0}_{L_2}, \label{eq:OU_E0} \\
    E_1(f,g,K,h) &= \sum_{k=1}^K \qty(\prod_{j=k+1}^K \gamma_{j,h}) e^{T-t_k} \nonumber\\
    &\qquad\quad \cdot \Bigg( L_1 h \big( 1 + \theta(T) + \omega(T) \big) \qty(e^h -1) \nonumber\\
    &\qquad\qquad+ \sqrt{h} \nu_{k,h} \qty(\intk e^{2(t_k-t)} L^2(T-t) \de t)^{\half}\Bigg), \label{eq:OU_E1}\\
    E_2(f,g,K,h,\mathcal{E}) &= \sum_{k=1}^K \qty(\prod_{j=k+1}^K \gamma_{j,h}) e^{T-t_k} \mathcal{E} \qty(e^h -1). \label{eq:OU_E2}
\end{align}
To prove Theorem~\ref{thm:err_bound_OU}, we further simplify these terms in order to arrive at an interpretable error bound clearly indicating the dependence on the parameters $T$, $h$, and $\mathcal{E}$.

\begin{proof}[Proof of Theorem~\ref{thm:err_bound_OU}]
    Using the substitution $B(t) = 1 + \alpha_0 (e^{2t} - 1)$, we can write
    \begin{align*}
        &\int_{t_k}^{T} \alpha(T-t) - M(T-t) \de t\\
        &\quad= \int_{0}^{T-t_k} \alpha(t) - M(t) \de t\\
        &\quad= \int_{0}^{T-t_k} \frac{\alpha_0e^{2t}}{1 + \alpha_0\qty(e^{2t}-1)} - \frac{M_0 e^{2t}}{\qty(1 + \alpha_0\qty(e^{2t}-1))^2} \de t\\
        &\quad= \half \int_{0}^{T-t_k} \frac{B'(t)}{B(t)} - \frac{M_0}{\alpha_0}\frac{B'(t)}{B^2(t)} \de t\\
        &\quad= \half \qty[\log(B(T-t_k)) - \log(B(0)) - \frac{M_0}{\alpha_0} \qty(\frac{1}{B(0)} - \frac{1}{B(T-t_k)})]\\
        &\quad= \half \log \qty(1 + \alpha_0\qty(e^{2(T-t_k)}-1)) - \frac{M_0}{2\alpha_0} \qty(1 - \frac{1}{1 + \alpha_0\qty(e^{2(T-t_k)}-1)}) \numberthis \label{eq:OU_alpha_M}
    \end{align*}
    For the initialization error \eqref{eq:OU_E0}, this yields
    \begin{align*}
        E_0(f,g,T) &\le C(\alpha_0, M_0) e^{- \int_0^T \alpha(t) - M(t) \de t} \norm{X_0}_{L_2}\\
        &= C(\alpha_0, M_0) \qty(1 + \alpha_0\qty(e^{2T}-1))^{-\half} e^{ \frac{M_0}{2\alpha_0} \qty(1 - \frac{1}{1 + \alpha_0\qty(e^{2T}-1)})} \norm{X_0}_{L_2}\\
        &= \mathcal{O}\qty(e^{-T} \norm{X_0}_{L_2})
    \end{align*}

    Next, we turn to the discretization error \eqref{eq:OU_E1}.
    According to the definitions \eqref{eq:omega(T)} and \eqref{eq:eta}, we have
    \begin{equation} \label{eq:OU_omega}
        \omega(T) = \sup_{0 \leq t \leq T} \left( e^{-2t} \|X_0\|^2_{L^2} + d (1 - e^{-2t}) \right)^\frac 12 \leq \|X_0\|_{L^2} + \sqrt d
    \end{equation}
    and
    \begin{equation*}    
        \eta(T) = \sup_{0 \leq t \leq T } \min \left\{
            e^{2t}, \frac 1{1 - e^{-2t}}
            \right \}  
            = \begin{cases}
                e^{2T} & T < \log \sqrt 2 \, \\
                2 & T \geq \log \sqrt 2 \, 
            \end{cases} \leq 2.
    \end{equation*}
    By Lemma~\ref{lem:kt_inf} and \ref{lem:Lt_bound}, it follows that
    \begin{align*}
        \sup_{0 \le t \le T} L(t) &\le \max\left\{(L_0 \vee 1) \eta(T), \frac{\abs{\alpha_0 - M_0}}{\alpha_0^2 \wedge 1} \xi(T) \right\} \le 2 \mathfrak a_0, \numberthis \label{eq:OU_L(t)_bound}\\
        \sup_{0 \leq t \leq T} -(\alpha(t) - M(t)) &\leq \frac{|\alpha_0 - M_0|}{\alpha_0^2 \wedge 1} \xi(T) \leq 2 \mathfrak a_0, \label{OU_K(t)_bound}
    \end{align*}
    where we define
    \begin{equation*}
        \mathfrak a_0 \coloneqq \max\left\{(L_0 \vee 1), \frac{\abs{\alpha_0 - M_0}}{\alpha_0^2 \wedge 1}\right\}. %
    \end{equation*}
    The upper bound for $L(t)$ in \eqref{eq:OU_L(t)_bound} together with the definition of $\nu_{k,h}$ in \eqref{eq:nu_kh} as well as Lemma~\ref{lem:theta_bound} further yields
    \begin{align*}
        \nu_{k,h} &= h \cdot \Big( \big(\theta(T) + \omega(T)\big) \big( 1 + L(T - t_k)\big) + \,  L_1(T+h) + \norm{\nabla \log p_0(\boldsymbol{0})} \Big)\\ 
        &\leq h \cdot \Big( \big(\sqrt{C(\alpha_0, M_0)} \norm{X_0}_{L_2} + \|X_0\|_{L^2} + \sqrt d\big) \big(1 + 2 \mathfrak a_0\big) + L_1(T + h) + \norm{\nabla \log p_0(\boldsymbol{0})}\Big)\\
        &= \mathcal{O}\qty( h \big(\norm{X_0}_{L_2} + \sqrt d + T\big)) \numberthis \label{eq:OU_nu}
    \end{align*}
    and
    \begin{equation} \label{eq:OU_E1_int}
        \intk e^{2(t_k-t)} L^2(T-t) \de t \le 4\mathfrak{a}_0^2 e^{2t_k} \intk e^{-2t} \de t = 2 \mathfrak{a}_0^2 \qty(e^{2h} - 1)
    \end{equation}
    Moreover, since $1-x \le e^{-x}$ for all $x$, we have
    \begin{align*}
        \prod_{j=k+1}^K \gamma_{j,h} %
        &= \prod_{j=k+1}^K \qty(1 - \intj \delta_j(T-t) \de t + L_1 h^2)\\
        &\le \prod_{j=k+1}^K \exp\qty(- \intj \delta_j(T-t) \de t + L_1 h^2)\\
        &= \exp\qty(- \sum_{j=k+1}^K \intj \delta_j(T-t) \de t) \exp\qty((K-k)L_1h^2) \numberthis \label{eq:OU_gamma}
    \end{align*}
    Further, we can compute
    \begin{align*}
        \sum_{j=k+1}^K \intj \delta_j(T-t) \de t &= \sum_{j=k+1}^K \intj e^{-(t-t_{j-1})} \big(\alpha(T-t) - M(T-t)\big) - \half h L^2(T-t) \de t\\
        &= \sum_{j=k+1}^K \intj e^{-(t-t_{j-1})} \big(\alpha(T-t) - M(T-t)\big) \de t\\
        &\qquad\quad - \half h \int_{t_k}^{T} L^2(T-t) \de t\\
        &\ge e^{-h} \int_{t_k}^{T} \alpha(T-t) - M(T-t) \de t - \half h \int_{t_k}^{T} L^2(T-t) \de t. \numberthis \label{eq:OU_delta}
    \end{align*}
    Combining \eqref{eq:OU_alpha_M}, \eqref{eq:OU_gamma}, \eqref{eq:OU_delta} and using the upper bound for $L(t)$ given in \eqref{eq:OU_L(t)_bound}, we get
    \begin{align*}
        \prod_{j=k+1}^K \gamma_{j,h} &\le \exp\qty(-e^{-h} \half \log \qty(1 + \alpha_0\qty(e^{2(T-t_k)}-1)) - e^{-h} \frac{M_0}{2\alpha_0} \qty(1 - \frac{1}{1 + \alpha_0\qty(e^{2(T-t_k)}-1)}))\\
        &\qquad \cdot \exp\qty(\half h \int_{t_k}^{T} L^2(T-t) \de t + (K-k) L_1 h^2)\\
        &\le \qty(1 + \alpha_0\qty(e^{2(T-t_k)}-1))^{-\half e^{-h}} \cdot \exp\qty(-e^{-h} \frac{M_0}{2\alpha_0} \qty(1 - \frac{1}{1 + \alpha_0\qty(e^{2(T-t_k)}-1)}))\\
        &\qquad \cdot \exp\qty(\mathfrak{a}_0 (T-t_k) h + L_1 (T-t_k) h)\\
        &= \mathcal{O}\qty(e^{-(T-t_k)e^{-h}} \cdot \exp(-e^{-h} \qty(1 - \frac{1}{e^{T-t_k}})) \cdot \exp\qty((T-t_k)h)).
    \end{align*}
    \allowdisplaybreaks
    From this result together with the upper bounds given in Lemma~\ref{lem:theta_bound}, \eqref{eq:OU_omega}, \eqref{eq:OU_nu}, and \eqref{eq:OU_E1_int}, it follows for the discretization error \eqref{eq:OU_E1} that
    \begin{align*}
        &E_1(f,g,K,h)\\
        &\quad\le \sum_{k=1}^K \mathcal{O}\qty(e^{-(T-t_k)e^{-h}} \cdot \exp(-e^{-h} \qty(1 - \frac{1}{e^{T-t_k}})) \cdot \exp\qty((T-t_k)h)) e^{T-t_k}\\
        &\quad\qquad\quad \cdot \Bigg( L_1 h \big( 1 + \mathcal{O}(\norm{X_0}_{L_2}) + \mathcal{O}(\norm{X_0}_{L_2} + \sqrt{d}) \big) \qty(e^h -1)\\
        &\quad\qquad\qquad+ \sqrt{h} \mathcal{O}\qty( h \big(\norm{X_0}_{L_2} + \sqrt d + T\big)) \Big(\mathcal{O}\qty(e^{2h} - 1)\Big)^{\half}\Bigg).\\
        &\quad= \sum_{k=1}^K \mathcal{O}\qty(e^{(T-t_k)(1-e^{-h})} \cdot \exp(-e^{-h} \qty(1 - \frac{1}{e^{T-t_k}})) \cdot e^{(T-t_k)h})\\
        &\quad\qquad\quad \cdot \qty(\mathcal{O}\qty(h (\norm{X_0}_{L_2} + \sqrt{d}) (e^h -1)) + \sqrt{h} \mathcal{O}\qty( h \big(\norm{X_0}_{L_2} + \sqrt d + T\big)) \Big(\mathcal{O}\qty(e^{2h} - 1)\Big)^{\half})\\
        &\quad\le \mathcal{O}\qty(K \cdot e^{T(1-e^{-h})} \cdot \exp\qty(-e^{-h} \qty(1 - e^{-T})) \cdot e^{Th})\\
        &\quad\qquad\quad \cdot \qty(\mathcal{O}\qty(h (\norm{X_0}_{L_2} + \sqrt{d}) (e^h -1)) + \sqrt{h} \mathcal{O}\qty( h \big(\norm{X_0}_{L_2} + \sqrt d + T\big)) \Big(\mathcal{O}\qty(e^{2h} - 1)\Big)^{\half}).
    \end{align*}
    The fact that $\mathcal{O}\qty(e^{ah}-1) = \mathcal{O}(h)$ for any $a > 0$ and $\mathcal{O}\qty(1 - e^{-T}) = \mathcal{O}(1)$ further simplifies the expression on the right-hand side, finally yielding
    \begin{align*}
        E_1(f, g, K, h) &\le \mathcal{O}\qty(K \cdot e^{Th} \cdot \exp\qty(-e^{-h}) \cdot e^{Th}) \cdot \mathcal{O}\qty(h^2 \big(\norm{X_0}_{L_2} + \sqrt d + T\big)))\\
        &= \mathcal{O}\qty(e^{Th} Th \qty(\norm{X_0}_{L_2}+\sqrt{d}+T)).
    \end{align*}
    
    Similarly, we get for the propagated score matching error \eqref{eq:OU_E2}
    \begin{align*}
        E_2(f,g,K,h,M) &= \mathcal{O}\qty(K \cdot e^{T(1-e^{-h})} \cdot \exp\qty(e^{-h} \qty(e^{-T} - 1)) \cdot e^{Th}) \mathcal{E} \qty(e^h -1)\\
        &= \mathcal{O}\qty(e^{Th} T \mathcal{E}). \qedhere
    \end{align*}
\end{proof}

\section{Interpretation of the main result}
\label{sec:app_interpret}

As $\gamma_{k,h}$ plays the role of a contraction rate for the discretization and propagated score matching error, i.e.\ the $L_2$-distance between $Y_{t_k}$ and $\widehat{Z}_{t_k}$ (see Proposition~\ref{prop:yt-zt-bound}), it is crucial to investigate whether or when it lies between 0 and 1. The following proposition establishes a regime shifting result (similar to Proposition~\ref{prop:regime-shift}) for this contraction rate.

\begin{proposition}[Regime shift for $\gamma_{k,h}$]
\label{prop:gamma_khT}
    Assuming %
    \begin{equation*}
        h < \bar{h} \coloneqq \min \left\{ \frac{\log(2)}{\max_{0 \le t \le T} f(t)}, \min_{t > \tau(\alpha_0, M_0)} \left\{ \frac{\frac{1}{4} g^2(t) (\alpha(t) - M(t))}{\frac{1}{8} g^4(t) L^2(t) + \half L_1 g^2(t)} \right\} \right\},
    \end{equation*}
    we have
    \begin{equation*}
        \begin{cases}
            \gamma_{k,h} \in (0,1), & k \in \left\{1, 2, \dots, \left\lfloor  K - \frac{\tau(\alpha_0, M_0)}{h} \right\rfloor \right\}
            \\
            \gamma_{k,h} > 1, & k \in \left\{\left\lceil K - \frac{\tau(\alpha_0, M_0)}{h} \right\rceil + 1, \dots, K-1, K\right\}.
        \end{cases}
    \end{equation*}
    Moreover, it holds that for $\tilde{T} = (K+\ell)h$, $\tilde{\gamma}_{k+\ell,h} =  \gamma_{k, h}$.
\end{proposition}

\begin{proof}[Proof of Proposition~\ref{prop:gamma_khT}]
    To simplify notation, we write
    \begin{equation*}
        \gamma_{k,h} = 1 - \intkT \delta_k(t) \de t + \half L_1 h \intkT g^2(t) \de t
    \end{equation*}
    and
    \begin{equation*}
        \delta_k(t) = \half e^{-\int_t^{T-t_{k-1}} f(s) \de s} g^2(t) \big(\alpha(t) - M(t)\big) - \frac{1}{8} h g^4(t) L^2(t).
    \end{equation*}
    By definition of the regime shift, if $t < \tau(\alpha_0, M_0)$ then $\alpha(t) - M(t) < 0$, and hence $\delta_k(t) < 0$. It follows that $\gamma_{k,h} > 1$ for all $k$ with $T - t_{k-1} \le \tau(\alpha_0, M_0)$, i.e.\ $k \ge K + 1 - h^{-1} \tau(\alpha_0, M_0)$.

    On the other hand, assume that $k \le K - h^{-1} \tau(\alpha_0, M_0)$ and thus $T-t_k \ge \tau(\alpha_0, M_0)$. Note that $h \le \bar{h}$ implies that $e^{-h \max_{0 \le t \le T} f(t)} > \half $ and thus 
    \begin{equation*}
        h < \min_{t > \tau(\alpha_0, M_0)} \left\{ \frac{\half e^{-h \max_{0 \le t \le T} f(t)} g^2(t) (\alpha(t)-M(t))}{\frac{1}{8} g^4(t) L^2(t) + \half L_1 g^2(t)} \right\}.
    \end{equation*}
    It follows that for $t > \tau(\alpha_0, M_0)$, in particular $t \in [T-t_k, T-t_{k-1}]$, it holds
    \begin{equation*}
        \delta_k(t) 
        > h \cdot \qty(\frac{1}{8} g^4(t) L^2(t) + \half L_1 g^2(t)) - \frac{1}{8} h g^4(t) L^2(t)
        = \half L_1 h g^2(t),
    \end{equation*}
    and hence $\gamma_{k,h} < 1$. 
    Moreover, inequality \eqref{eq:S1_bound} in the proof of Proposition~\ref{prop:yt-zt-bound} implies that
    \begin{equation*}
        1 - \intk \delta_k(T-t) \de t \ge 0.
    \end{equation*}
    Consequently, $\gamma_{k,h} > 0$ since we assumed $g$ to be positive for all $t > 0$.

    Now, let $\tilde{T} = (K+\ell)h$. Then we have $\tilde{T} - t_{k+\ell} = T - t_k$ and hence
    \begin{align*}
        \tilde{\gamma}_{k+\ell,h} &= 1 - \int_{\tilde{T} - t_{k+\ell}}^{\tilde{T} - t_{k+\ell-1}} \tilde{\delta}_{k+\ell}(t) \de t + \half L_1 h \int_{\tilde{T} - t_{k+\ell}}^{\tilde{T} - t_{k+\ell-1}} g^2(t) \de t\\
        &= 1 - \intkT \tilde{\delta}_{k+\ell}(t) \de t + \half L_1 h \intkT g^2(t) \de t,
    \end{align*}
    and
    \begin{align*}
        \tilde{\delta}_{k+\ell}(t) &= \half e^{-\int_t^{\tilde{T}-t_{k+\ell-1}} f(s) \de s} g^2(t) \big(\alpha(t) - M(t)\big) - \frac{1}{8} h g^4(t) L^2(t)\\
        &= \half e^{-\int_t^{T-t_{k-1}} f(s) \de s} g^2(t) \big(\alpha(t) - M(t)\big) - \frac{1}{8} h g^4(t) L^2(t)\\
        &= \delta_k(t),
    \end{align*}
    which completes the proof.
\end{proof}

Note that, if $\tau(\alpha_0, M_0)$ is not evenly divisible by $h$, it is not clear whether $\gamma_{k,h}$ will be less or greater than one for $k = \left\lceil K - \frac{\tau(\alpha_0, M_0)}{h} \right\rceil$.
The second part of Proposition~\ref{prop:gamma_khT} means that, when increasing $T=Kh$ to $\tilde{T} = (K+\ell)h$ for some integer $\ell \ge 1$, we have
\begin{equation*}
    \prod_{k=1}^{K+\ell} \tilde{\gamma}_{k,h} < \prod_{k=1}^K \gamma_{k,h},
\end{equation*}
which lies at the core of the discretization error $E_1(f,g,K,h)$ defined in \eqref{eq:E1}.

The following lemma provides an upper bound for $\theta(T)$, another term involved in the discretization error $E_1$. It is used when
comparing our error bound to the strongly log-concave case in Section~\ref{sec:comp_strong} and to establish the more accessible error bound for the OU process in the proof of Theorem~\ref{thm:err_bound_OU}.

\begin{lemma}[Upper bound for $\theta(T)$]
\label{lem:theta_bound}
    It holds that
    \begin{equation*}
        \theta(T) \le \sqrt{C(\alpha_0, M_0)} \norm{X_0}_{L_2},
    \end{equation*}
    where $C(\alpha_0, M_0, T)$ is defined in \eqref{eq:C(alpha, M)}.
\end{lemma}
\begin{proof}[Proof of Lemma~\ref{lem:theta_bound}]
    By the definition of $\theta(T)$ in \eqref{eq:theta(T)} and the non-negativity of $f(t)$, we have
    \begin{align*}
        \theta(T) &= \sup_{0 \le t \le T} e^{-\half \int_0^t g^2(T-s)(\alpha(T-s)-M(T-s)) - 2f(T-s) \de s} e^{- \int_0^T f(T-s) \de s} \norm{X_0}_{L_2}\\
        &= \sup_{0 \le t \le T} e^{-\half \int_0^t g^2(T-s)(\alpha(T-s)-M(T-s)) \de s} e^{- \int_t^T f(T-s) \de s} \norm{X_0}_{L_2}\\
        &\le \sup_{0 \le t \le T} e^{-\half \int_{T-t}^T g^2(s)(\alpha(s)-M(s)) \de s} \norm{X_0}_{L_2}\\
        &= \sup_{0 \le t \le T} e^{-\half \int_{t}^T g^2(s)(\alpha(s)-M(s)) \de s} \norm{X_0}_{L_2}.
    \end{align*}
    Since $\alpha(s) - M(s) > 0$ for any $s > \tau(\alpha_0, M_0)$, it follows that
    \begin{equation*}
        e^{-\half \int_{t}^T g^2(s)(\alpha(s)-M(s)) \de s} < 1
    \end{equation*}
    for all $t > \tau(\alpha_0, M_0)$, and therefore
    \begin{align*}
        \theta(T) &\le \max\left\{1, \sup_{0 \le t \le \tau(\alpha_0, M_0)} e^{-\half \int_{t}^{T} g^2(s)(\alpha(s)-M(s)) \de s} \right\} \norm{X_0}_{L_2}\\
        &= \max\left\{1, \sup_{0 \le t \le \tau(\alpha_0, M_0)} e^{-\half \int_{t}^{\tau(\alpha_0, M_0)} g^2(s)(\alpha(s)-M(s)) \de s} e^{-\half \int_{\tau(\alpha_0, M_0)}^T g^2(s)(\alpha(s)-M(s)) \de s} \right\} \norm{X_0}_{L_2}\\
        &\le \max\left\{1, \sup_{0 \le t \le \tau(\alpha_0, M_0)} e^{-\half \int_{t}^{\tau(\alpha_0, M_0)} g^2(s)(\alpha(s)-M(s)) \de s} \right\} \norm{X_0}_{L_2}\\
        &= \max\left\{1, \sup_{0 \le t \le \tau(\alpha_0, M_0)} e^{\half \int_{t}^{\tau(\alpha_0, M_0)} g^2(s) \abs{\alpha(s)-M(s)} \de s} \right\} \norm{X_0}_{L_2}\\
        &= e^{\half \int_{0}^{\tau(\alpha_0, M_0)} g^2(s) \abs{\alpha(s)-M(s)} \de s} \norm{X_0}_{L_2}.
    \end{align*}
    Using Lemma~\ref{lem:kt_inf}, we can further say that
    \begin{equation*}
        \theta(T) \le e^{\half \frac{\abs{\alpha_0 - M_0}}{\alpha_0^2 \wedge 1} \xi(\tau(\alpha_0, M_0)) \int_0^{\tau(\alpha_0, M_0)} g^2(s) \de s} \norm{X_0}_{L_2} = \sqrt{C(\alpha_0, M_0, T)} \norm{X_0}_{L_2}. \qedhere
    \end{equation*}
\end{proof}

Next, we provide the proof of Proposition~\ref{prop:asymp_same}, establishing the remarkable finding that the asymptotics of our error bound given in Theorem~\ref{thm:main_thm} are the same as under the stricter assumption of strong log-concavity.

\begin{proof}[Proof of Proposition~\ref{prop:asymp_same}]
    To analyze the differences in the asymptotics with respect to $T$, $h$, and $\mathcal{E}$ of the bound in Theorem~\ref{thm:main_thm} if $p_0$ is only weakly log-concave compared to the strongly log-concave case analyzed in \citet[Theorem~2]{gao2024convergence}, we just need to consider the consequences of the differences in the error bounds as listed in Section~\ref{sec:comp_strong}. We discuss the effect of each difference point-by-point.
    \begin{enumerate}
        \item The constant $C(\alpha_0, M_0)$ does not influence the asymptotics. For the exponential term in the initialization error, we have
        \begin{align*}
            e^{- \int_0^T g^2(t)|\alpha(t) - M(t)| \de t} &\le e^{-\int_0^T g^2(t) \alpha(t) \de t} \cdot e^{\int_0^T g^2(t) M(t) \de t}.
        \end{align*}
        So, in order to identify the difference to the strongly log-concave case, we need to analyze the second coefficient involving $M(t)$. For a VE-SDE, i.e.\ $f(t) \equiv 0$, we have
        \begin{align*}
            \int_0^T g^2(t) M(t) \de t &= \int_0^T \frac{M_0 g^2(t)}{\qty(1+\alpha_0\int_0^t g^2(s) \de s)^2} \de t\\
            &= - \frac{M_0}{2 \alpha_0} \qty(\frac{1}{1 + \alpha_0\int_0^T g^2(s) \de s}-1)\\
            &= \mathcal{O}\qty(1 - \frac{1}{\int_0^T g^2(s) \de s}) = \mathcal{O}(1), \numberthis \label{eq:int_M_VE}
        \end{align*}
        where we used the fact that $g$ is positive and $T$ diverges.
        In the VP case, i.e.\ $f(t) = \half \beta(t)$ and $g(t) = \sqrt{\beta(t)}$, on the other hand, it follows from the substitution $B(t) = 1+\alpha_0(e^{\mathcal{B}(t)} - 1)$ that
        \begin{align*}
            \int_0^T g^2(t) M(t) \de t &= \int_0^T \frac{M_0 \beta(t) e^{\mathcal{B}(t)}}{\qty(1+\alpha_0(e^{\mathcal{B}(t)} - 1))^2} \de t\\
            &= \frac{M_0}{\alpha_0} \int_0^T \frac{B'(t)}{B^2(t)} \de t\\
            &= -\frac{M_0}{\alpha_0} \qty(\frac{1}{B(T)} - \frac{1}{B(0)}) \de t\\
            &= -\frac{M_0}{\alpha_0} \qty(\frac{1}{1 + \alpha_0(e^{\mathcal{B}(T)} - 1)}-1) \\
            &= \mathcal{O}\qty(1 - e^{-\mathcal{B}(T)}) = \mathcal{O}(1), \numberthis \label{eq:int_M_VP}
        \end{align*}
        where we reused the definition $\mathcal{B}(t) = \int_0^t \beta(s) \de s$ and the value of $M(t)$ given in \eqref{eq:kt-vp} from Example~\ref{ex:kt_vp}.
        In both cases, there is no change in the asymptotics.
        
        \item To determine how the change in $\delta_k(T-t)$ influences the limit behavior, we recapitulate how it is analyzed in \citet[Corollary 6-9]{gao2024convergence}. The coefficient $\prod_{j=k+1}^K \gamma_{j,h}$ in the error bound given in Theorem~\ref{thm:main_thm} is upper bounded using the fact that $1-x \le e^{-x}$ for all $x \in \mathbb{R}$. Accordingly, we have
        \begin{align*}
            \prod_{j=k+1}^K \gamma_{j,h} &= \prod_{j=k+1}^K 1 - \intj \delta_j(T-t) \de t + \half L_1 h \intj g^2(T-t) \de t\\
            &\le \prod_{j=k+1}^K \exp\qty(- \intj \delta_j(T-t) \de t + \half L_1 h \intj g^2(T-t) \de t)\\
            &= \exp\qty(- \sum_{j=k+1}^K \intj \delta_j(T-t) \de t + \half L_1 h \int_{t_k}^{T} g^2(T-t) \de t).
        \end{align*}
        The only new term in the above display emerging in the weak log-concave case is 
        \begin{align*}
            &\exp\qty(\sum_{j=k+1}^K \intj \half e^{-\int_{t_{j-1}}^t f(T-s) \de s} g^2(T-t) M(T-t) \de t)\\
            &\quad\le \exp\qty(\half \int_{t_k}^{T} g^2(T-t) M(T-t) \de t)\\
            &\quad\le \exp\qty(\half \int_{0}^{T} g^2(t) M(t) \de t),
        \end{align*}
        where we used the non-negativity of $f(t)$ and $M(t)$. Both, in the VE and VP case, the term on the far right-hand side is in $\mathcal{O}(1)$ as shown in \eqref{eq:int_M_VE} and \eqref{eq:int_M_VP}. Thus, the asymptotic behavior remains unchanged. For a discussion of $\theta(T)$, see point \ref{enum:theta}.
        
        \item When analyzing the limit behavior in \citet[Corollary 6-9]{gao2024convergence}, the time-dependent Lipschitz constant is dealt with by finding an upper bound for $L(t)$ in the VP case and $g^2(t)L(t)$ in the VE case. Denote the upper bound for the Lipschitz constant $L^{GZ}(t)$ in Gao and Zhu's paper by $\bar{L}^{GZ}$. Note that we have
        \begin{equation*}
            L(t) = \max \left\{ L^{GZ}(t), -\qty(\alpha(t)-M(t)) \right\},
        \end{equation*}
        so it suffices to show that $-\qty(\alpha(t)-M(t))$ is appropriately bounded. By Lemma~\ref{lem:kt_inf}, we have
        \begin{align*}
            -\qty(\alpha(t)-M(t)) &\le \frac{|\alpha_0 - M_0|}{\alpha_0^2 \wedge 1} \min \left\{ e^{ 2 \int_0^{t} f(s) ds}, \frac{ e^{ 2 \int_0^{t} f(s) ds}}{(\int_0^{t} e^{2\int_0^{s} f(v) \de v} g^2(s) \de s)^2} \right\}\\
            &\le \frac{|\alpha_0 - M_0|}{\alpha_0^2 \wedge 1} \min \left\{ e^{2 \int_0^{t} f(s) ds}, \frac{1}{\int_0^{t} e^{-2\int_s^{t} f(v)dv} g^2(s) ds} \right\},
        \end{align*}
        where the last inequality follows from the arguments in Lemma~\ref{lem:Lt_bound}. Since
        \begin{equation*}
            \bar{L}^{GZ} \ge L^{GZ}(t) \ge (L_0 \wedge 1) \min \left\{ e^{2 \int_0^{t} f(s) ds}, \frac{1}{\int_0^{t} e^{-2\int_s^{t} f(v)dv} g^2(s) ds} \right\},
        \end{equation*}
        it follows that
        \begin{equation*}
            -\qty(\alpha(t)-M(t)) \le \frac{|\alpha_0 - M_0|}{\alpha_0^2 \wedge 1} \frac{1}{L_0 \wedge 1} \bar{L}^{GZ},
        \end{equation*}
        and hence
        \begin{equation*}
            L(t) \le \max\left\{1, -\frac{|\alpha_0 - M_0|}{\alpha_0^2 \wedge 1} \frac{1}{L_0 \wedge 1}\right\} \bar{L}^{GZ}.
        \end{equation*}
        Similar arguments lead to an upper bound for $g^2(t)L(t)$. As the bound only differs by some coefficient that is independent of $T$, $h$, and $\mathcal{E}$, the asymptotics are not affected.
        
        \item The difference of the coefficients does not have any effect on the asymptotics.
        
        \item Since $h = o(1)$, we have $\mathcal{O}(T+h) = \mathcal{O}(T)$.
        
        \item \label{enum:theta} The constant coefficient $\sqrt{C(\alpha_0, M_0)}$ does not influence the asymptotics. \qedhere
    \end{enumerate}
\end{proof} %
\section{Proof of the main result}
\label{sec:app_proof}

As shown in Section~\ref{sec:proof}, the proof of Theorem~\ref{thm:main_thm} is based on Proposition~\ref{prop:yt-p0-bound} and \ref{prop:yt-zt-bound}, splitting the overall error $\mathcal{W}_2(\mathcal{L}(\widehat{Z}_T), p_0)$ into the initialization error $\mathcal W_2(\mathcal{L}(Y_T), p_0)$ and the combined discretization and propagated score-matching error $\mathcal W_2(\mathcal{L}(\widehat{Z}_T), \mathcal{L}(Y_T))$. Here, we provide the proofs of the two propositions.

\subsection{Proof of Proposition~\ref{prop:yt-p0-bound}}

We start by analyzing the initialization error.
Recall the following result from \citet[Lemma~16]{gao2024convergence}.
\begin{lemma}\label{lem:w2-bound-ini}
    
    It holds that 
    \[
\mathcal W_2\bigl(p_T,\hat p_T\bigr)
\;\le\;
e^{-\int_{0}^{T} f(s) \de s }\,
\|X_0\|_{L_2}.
\]
\end{lemma}

\begin{proof}[Proof of Proposition~\ref{prop:yt-p0-bound}]
The result is a consequence of the propagation over time of the weak log-concavity, combined with the regime change results from Section~\ref{sec:weak_conv_prop}.
We start by following the steps in \citet[Proposition~14]{gao2024convergence}. 
Let 
\[
m(t) = -2 f(t) + g^2(t) (\alpha(t) - M(t)).
\]
By computing the derivative, using \eqref{eq:flow} and \eqref{eq:flow-hat}, and by Proposition~\ref{prop:weak_pt}, we get
\begin{align*}
    \frac{\de }{\de t} 
   & \left( \lVert \tilde{X}_t - Y_t\rVert^{2}\,
               e^{\int_{0}^{t} m(T-s) \de s} \right) 
               \\
&= m(T-t)\,e^{\int_{0}^{t} m(T-s)\,\mathrm{d}s}\,
 \lVert \tilde{X}_t - Y_t\rVert^{2}\,
 + 2\,e^{\int_{0}^{t} m(T-s)\,\mathrm{d}s}\,
   \bigl\langle \tilde{X}_t - Y_t,\,
                 \frac{\de }{\de t} \tilde{X}_t -  \frac{\de }{\de t} Y_t\bigr\rangle \\
&= m(T-t)\,e^{\int_{0}^{t} m(T-s)\,\mathrm{d}s}\,
 \lVert \tilde{X}_t - Y_t\rVert^{2} 
 + 2\,e^{\int_{0}^{t} m(T-s)\,\mathrm{d}s}\,
   \bigl\langle \tilde{X}_t - Y_t,\,
                f(T-t)\bigl(\tilde{X}_t - Y_t\bigr)\bigr\rangle\,
   \\
&\quad + 2\,e^{\int_{0}^{t} m(T-s)\,\mathrm{d}s}\,
 \Bigl\langle \tilde{X}_t - Y_t,\,
   \tfrac12 g^2(T-t)
   \bigl(\nabla\log p_{T-t}(\tilde{X}_t) -
         \nabla\log p_{T-t}(Y_t)\bigr)\Bigr\rangle\,
 \\
&\le
 e^{\int_{0}^{t} m(T-s)\,\mathrm{d}s}\,\lVert \tilde{X}_t - Y_t\rVert^{2}
 \left[ m(T-t)+2f(T-t)- g^2 (T-t)(\alpha(T-t) - M(T-t)) \right]     \\ 
&= 0 .
\end{align*}
Hence, for any $t \in [0,T]$, 
\begin{equation}
\label{eq:xt-yt_x0-y0}
    \|\tilde{X}_t - Y_t\|^2 e^{\int_0^t m(T-s) \, ds} \leq \|\tilde X_0 - Y_0\|^2,
\end{equation}
so that
\[
\mathbb{E} \|\tilde{X}_T - Y_T\|^2 \leq e^{-\int_0^T m(T-s) \, ds} \mathbb{E} \|\tilde X_0 - Y_0\|^2.
\]

Next, consider a coupling of \((\tilde X_0, Y_0)\) such that \(\tilde X_0 \sim p_T\), \(Y_0 \sim \hat{p}_T\), and \(\mathbb{E} \|\tilde X_0 - Y_0\|^2 = \mathcal W_2^2(p_T, \hat{p}_T)\). By combining the previous result with Lemma~\ref{lem:w2-bound-ini} and by the definition of the Wasserstein distance~\eqref{eq:wass-def}, we have 
\begin{align*}
    \mathcal W_2^2(\mathcal{L}(Y_T), p_0) &= W_2^2(\mathcal{L}(Y_T), \mathcal{L}(\tilde{X}_T)) \leq \mathbb{E} \|\tilde{X}_T - Y_T\|^2 \\
    &\leq e^{-\int_0^T m(T-s) \, \de s} \mathcal W_2^2(p_T, \hat{p}_T) \\
    &\leq e^{-\int_0^T m(s) \de s} e^{-2 \int_0^T f(s) \de s} \|X_0\|_{L_2}^2 \\
    &= e^{- \int_0^T g^2(t) (\alpha(t) - M(t)) \de t} \|X_0\|_{L_2}^2. \numberthis \label{eq:W2_yt-p0}
\end{align*}
Recall the regime shift result from Proposition~\ref{prop:regime-shift}:
\[
\begin{cases}
    \alpha(t) - M(t) <  0 & 0 < t < \tau(\alpha_0, M_0) \wedge T \\
     \alpha(t) - M(t) \geq  0 & \tau(\alpha_0, M_0) \wedge T \leq  t < T. 
\end{cases}\label{eq:alphaM-cases}
\]
From this and Lemma~\ref{lem:kt_inf}, we get 
\begin{align*}
&\exp\left( - \int_0^T g^2(t) (\alpha(t) - M(t))  \de t \right) \\
&\quad = \exp \left(- \int_0^T g^2(t) |\alpha(t) - M(t)|  \de t + 2\int_0^{\tau(\alpha_0, M_0) \wedge T}  g^2(t) |\alpha(t) - M(t)| \de t \right)
\\
&\quad  \leq 
\exp \left(- \int_0^{ T} g^2(t) |\alpha(t) - M(t)|  \de t + 2
\sup_{0 
\leq t \leq \tau(\alpha_0, M_0) } |\alpha(t) - M(t)|
\int_0^{\tau(\alpha_0, M_0)} g^2(t) \de t \right)
\\ &\quad  = 
\exp \left(- \int_0^{ T} g^2(t) |\alpha(t) - M(t)|  \de t - 2
\inf_{0 
\leq t \leq \tau(\alpha_0, M_0) } K(t)
\int_0^{\tau(\alpha_0, M_0)} g^2(t) \de t \right)
\\ &\quad \leq
\exp \left(- \int_0^{ T} g^2(t) |\alpha(t) - M(t)|  \de t + 
2 \frac{|\alpha_0 - M_0|}{\alpha_0^2 \wedge 1} \xi(\tau(\alpha_0, M_0))
\int_0^{\tau(\alpha_0, M_0)} g^2(t) \de t \right)
\\ &\quad =
C^2(\alpha_0, M_0) \exp \left(- \int_0^{ T} g^2(t) |\alpha(t) - M(t)|  \de t \right).
\end{align*}
Together with \eqref{eq:W2_yt-p0}, it follows that
\begin{equation*}
    \mathcal W_2^2(\mathcal{L}(Y_T), p_0) \le C^2(\alpha_0, M_0) \exp \left(- \int_0^{ T} g^2(t) |\alpha(t) - M(t)|  \de t \right) \|X_0\|_{L_2}^2.
\end{equation*}
We note that the quantity $C(\alpha_0, M_0)$ is always finite for any positive $\alpha_0$ and $M_0$, since $g$ is continuous and $\tau (\alpha_0, M_0)$ is finite.
\end{proof} 

\subsection{Proof of Proposition~\ref{prop:yt-zt-bound}}

Next, we examine the discretization and propagated score-matching error.
For that, we need two technical lemmas.
\begin{lemma}
\label{lem:omega(T)}
    With $\omega(T)$ defined in \eqref{eq:omega(T)}, it holds that
    $$\sup_{0 \le t \le T} \norm{X_t}_{L_2} = \omega(T).$$
\end{lemma}

\begin{proof}[Proof of Lemma~\ref{lem:omega(T)}]
    Using the explicit formula for $X_t$ given in \eqref{eq:forw_SDE_solution} and the distribution of the stochastic integral therein as well as its independence of $X_0$, we get
    \begin{align*}
        \norm{X_t}_{L_2}^2 &= \mathbb{E} \qty[\norm{e^{-\int_{0}^{t}f(s) \de s}\,X_0 + \int_{0}^{t}e^{-\int_{s}^{t}f(v) \de v}g(s) \de B_s}^2]\\
        &= \mathbb{E} \qty[\norm{e^{-\int_{0}^{t}f(s) \de s}\,X_0}^2] + \mathbb{E}\qty[\norm{\int_{0}^{t}e^{-\int_{s}^{t}f(v) \de v}g(s) \de B_s}^2]\\
        &= e^{-2\int_{0}^{t}f(s) \de s} \norm{X_0}_{L_2}^2 + \text{tr}\qty(\text{Var}\qty(\int_{0}^{t}e^{-\int_{s}^{t}f(v) \de v}g(s) \de B_s))\\
        &= e^{-2\int_{0}^{t}f(s) \de s} \norm{X_0}_{L_2}^2 + d \cdot \int_{0}^{t}e^{-2\int_{s}^{t}f(v) \de v}g^2(s) \de s. \qedhere
    \end{align*}
\end{proof}

\begin{lemma}
\label{lem:nu_kh}
    With $\nu_{k,h}$ defined in \eqref{eq:nu_kh}, it holds for any $k \in \{1,\dots,K\}$ that
    \begin{equation*}
        \sup_{t_{k-1} \le t \le t_k} \norm{Y_t - Y_{t_{k-1}}}_{L_2} \le \nu_{k,h}.
    \end{equation*}
\end{lemma}

\begin{proof}[Proof of Lemma~\ref{lem:nu_kh}] 
    From \eqref{eq:flow} and \eqref{eq:flow-hat}, it follows that for any $t \in [t_{k-1}, t_k]$
    \begin{align}
        \tilde{X}_t &= \tilde{X}_{t_{k-1}} + \int_{t_{k-1}}^t \qty[f(T-s) \tilde{X}_s + \half g^2(T-s) \nabla \log p_{T-s}(\tilde{X}_s)] \de s \label{eq:xt_interval}\\
        Y_t &= Y_{t_{k-1}} + \int_{t_{k-1}}^t \qty[f(T-s) Y_s + \half g^2(T-s) \nabla \log p_{T-s}(Y_s)] \de s, \nonumber,
    \end{align}
    so that, by an application of the triangle inequality,
    \begin{align*}
        \norm{Y_t - Y_{t_{k-1}}}_{L_2} &= \Bigg\lVert \tilde{X}_t - \tilde{X}_{t_{k-1}} + \int_{t_{k-1}}^t f(T-s) \qty(Y_s - \tilde{X}_s) \de s\\
        &\qquad+ \int_{t_{k-1}}^t \half g^2(T-s) \qty(\nabla \log p_{T-s}(Y_s) - \nabla \log p_{T-s}(\tilde{X}_s))\Bigg] \de s \Bigg\rVert_{L_2}\\
        &\le \norm{\tilde{X}_t - \tilde{X}_{t_{k-1}}}_{L_2} + \int_{t_{k-1}}^t f(T-s) \norm{Y_s - \tilde{X}_s}_{L_2} \de s\\
        &\qquad+ \int_{t_{k-1}}^t \half g^2(T-s) \norm{\nabla \log p_{T-s}(Y_s) - \nabla \log p_{T-s}(\tilde{X}_s)}_{L_2} \de s.
    \end{align*}
    An application of Proposition~\ref{prop:weak_Lip} further yields
    \begin{equation*}
        \norm{Y_t - Y_{t_{k-1}}}_{L_2} \le \norm{\tilde{X}_t - \tilde{X}_{t_{k-1}}}_{L_2} + \int_{t_{k-1}}^t \qty[f(T-s) + \half g^2(T-s) L(T-s)] \norm{Y_s - \tilde{X}_s}_{L_2} \de s.
    \end{equation*}
    From the proof of Proposition~\ref{prop:yt-p0-bound}, specifically \eqref{eq:xt-yt_x0-y0} and the lines thereafter, we have
    \begin{equation}
    \label{eq:yt-xt_bound}
        \norm{Y_t - \tilde{X}_t}_{L_2} \le e^{-\half \int_0^t m(T-s)\de s} \norm{Y_0 - \tilde{X}_0}_{L_2} \le e^{-\half \int_0^t m(T-s)\de s} e^{-\int_0^T f(s) \de s} \norm{X_0}_{L_2} \le \theta(T),
    \end{equation}
    and therefore 
    \begin{equation}
    \label{eq:yt-yk_split}
        \norm{Y_t - Y_{t_{k-1}}}_{L_2} \le \norm{X_t - X_{t_{k-1}}}_{L_2} + \theta(T) \int_{t_{k-1}}^t \qty[f(T-s) + \half g^2(T-s) L(T-s)] \de s.
    \end{equation}

    Next, \eqref{eq:xt_interval} implies that
    \begin{equation*}
        \norm{\tilde{X}_t - \tilde{X}_{t_{k-1}}}_{L_2} \le \int_{t_{k-1}}^t \qty[f(T-s) \norm{\tilde{X}_s}_{L_2} + \half g^2(T-s) \norm{\nabla \log p_{T-s}(\tilde{X}_s)}_{L_2}] \de s.
    \end{equation*}
    Another application of Proposition~\ref{prop:weak_Lip} and the fact that $p_{T-s}(\boldsymbol{0})$ is determistic yields
    \begin{align*}
        \norm{\nabla \log p_{T-s}(\tilde{X}_s)}_{L_2} &\le \norm{\nabla \log p_{T-s}(\tilde{X}_s) - \nabla \log p_{T-s}(\boldsymbol{0})}_{L_2} + \norm{\nabla \log p_{T-s}(\boldsymbol{0})}_{L_2}\\
        &\le L(T-s) \norm{\tilde{X}_s}_{L_2} + \norm{\nabla \log p_{T-s}(\boldsymbol{0})}
    \end{align*}
    Moreover, since $s \in [t_{k-1}, t_k]$ and $t_K=Kh=T$, it follows from Assumption \ref{ass:Lip_in_time} that
    \begin{align}
        \norm{\nabla \log p_{T-s}(\boldsymbol{0})} &\le \norm{\nabla \log p_{T-s}(\boldsymbol{0}) - \nabla \log p_{T-t_{k-1}}(\boldsymbol{0})} \nonumber\\
        &\quad + \sum_{j=k}^K \norm{\nabla \log p_{T-t_j}(\boldsymbol{0}) - \nabla \log p_{T-t_{j-1}}(\boldsymbol{0})} + \norm{\nabla \log p_0(\boldsymbol{0})} \nonumber\\
        &\le (K-k+2)L_1 h + \norm{\nabla \log p_0(\boldsymbol{0})} \nonumber\\
        &\le (K+1) L_1 h + \norm{\nabla \log p_0(\boldsymbol{0})} \nonumber\\
        &\le (T+h) L_1 + \norm{\nabla \log p_0(\boldsymbol{0})} \label{eq:bound_p_T-s(0)}.
    \end{align}
    In summary, we conclude that
    \begin{align}
        \norm{\tilde{X}_t - \tilde{X}_{t_{k-1}}}_{L_2} &\le \int_{t_{k-1}}^t \qty[f(T-s) + \half g^2(T-s) L(T-s)] \norm{\tilde{X}_s}_{L_2} \de s \nonumber\\
        &\quad + \half \Big((T+h) L_1 + \norm{\nabla \log p_0(\boldsymbol{0})}\Big) \int_{t_{k-1}}^t g^2(T-s) \de s. \label{eq:xt-xk}
    \end{align}
    The final result follows from a combination of \eqref{eq:yt-yk_split} and \eqref{eq:xt-xk} together with the observation that
    \begin{equation*}
        \sup_{0 \le s \le T} \norm{\tilde{X}_s}_{L_2} = \sup_{0 \le s \le T} \norm{X_{T-s}}_{L_2} = \omega(T),
    \end{equation*}
    where the first equality holds because $\tilde{X}_s = X_{T-s}$ in distribution and the second equality is verified in Lemma~\ref{lem:omega(T)}.
\end{proof}

Now, we are ready to prove Proposition~\ref{prop:yt-zt-bound}.

\begin{proof}[Proof of Proposition~\ref{prop:yt-zt-bound}]
    We follow the steps in \citet[Proposition~15]{gao2024convergence}. Specifically, we split the distance between $Y_{t_k}$ and $\widehat{Z}_{t_k}$ into several parts and derive upper bounds for each one of them separately, repeatedly making use of the propagation of weak log-concavity and Lipschitz-smoothness from $p_0$ to $p_t$ as established in Proposition~\ref{prop:weak_pt} and \ref{prop:weak_Lip}. 
    
    By the definition of $Y_t$ and $\widehat{Z}_{t}$ given in \eqref{eq:flow-hat} and \eqref{eq:flow_score}, we have for any $t \in [t_{k-1}, t_k]$
    \begin{align*}
        Y_t &= Y_{t_{k-1}} + \int_{t_{k-1}}^t \qty[f(T-s) Y_s + \half g^2(T-s) \nabla \log p_{T-s}(Y_s)] \de s,\\
        \widehat{Z}_t &= \widehat{Z}_{t_{k-1}} + \int_{t_{k-1}}^t \qty[f(T-s) \widehat{Z}_s + \half g^2(T-s) s_\theta \qty(\widehat{Z}_{t_{k-1}}, T-t_{k-1})] \de s,
    \end{align*}
    which yields the solutions
    \begin{align*}
        Y_{t_k} &= e^{\intk f(T-t) \de t} Y_{t_{k-1}} + \half \intk e^{\int_t^{t_k} f(T-s) \de s} g^2(T-t) \nabla \log p_{T-t}(Y_t) \de t,\\
        \widehat{Z}_{t_k} &= e^{\intk f(T-t) \de t} \widehat{Z}_{t_{k-1}} + \half \intk e^{\int_t^{t_k} f(T-s) \de s} g^2(T-t) s_\theta \qty(\widehat{Z}_{t_{k-1}}, T-t_{k-1}) \de t.
    \end{align*}
    By adding and subtracting some additional terms as well as several applications of the triangle inequality, it follows that
    \begin{align*}
        &\norm{Y_{t_k} - \widehat{Z}_{t_k}}_{L_2}\\
        &\quad\le \Bigg\lVert e^{\intk f(T-t) \de t} \qty(Y_{t_{k-1}} - \widehat{Z}_{t_{k-1}}) \\
        &\quad\qquad + \half\, \intk e^{\int_t^{t_k} f(T-s) \de s} g^2(T-t) \qty(\nabla \log p_{T-t}(Y_{t_{k-1}}) - \nabla \log p_{T-t}(\widehat{Z}_{t_{k-1}})) \de t \Bigg\rVert_{L_2} \\
        &\quad\quad + \half \norm{\intk e^{\int_t^{t_k} f(T-s) \de s} g^2(T-t) \qty(\nabla \log p_{T-t}(Y_t) - \nabla \log p_{T-t}(Y_{t_{k-1}})) \de t}_{L_2} \\
        &\quad\quad + \half\, \Bigg\lVert \intk e^{\int_t^{t_k} f(T-s) \de s} g^2(T-t) \qty(\nabla \log p_{T-t}(\widehat{Z}_{t_{k-1}}) - \nabla \log p_{T-t_{k-1}}(\widehat{Z}_{t_{k-1}})) \de t \Bigg\rVert_{L_2} \\
        &\quad\quad + \half\, \Bigg\lVert \intk e^{\int_t^{t_k} f(T-s) \de s} g^2(T-t) \\
        &\quad\qquad\qquad\qquad\quad \cdot \qty(\nabla \log p_{T-t_{k-1}}(\widehat{Z}_{t_{k-1}}) - s_\theta \qty(\widehat{Z}_{t_{k-1}}, T-t_{k-1})) \de t \Bigg\rVert_{L_2} \\
        &\quad\eqqcolon \norm{S_1(k,h)}_{L_2} + \half\,\norm{S_2(k,h)}_{L_2} + \half\,\norm{S_3(k,h)}_{L_2} + \half\,\norm{S_4(k,h)}_{L_2}. \numberthis \label{eq:err_bound_4_summands}
    \end{align*}
    Next, we derive upper bounds for the four summands $\norm{S_i(k,h)}_{L_2}$, $i \in \{1,\dots,4\}$, that appear in \eqref{eq:err_bound_4_summands}. For $S_1(k,h)$ and $S_2(k,h)$, we first derive an upper bound for the Euclidean norm and then deduct one for the $L_2$-norm.
    
    For the first term, we get
    \begin{align*}
        &\norm{S_1(k,h)}^2\\
        &\quad= e^{2 \intk f(T-t) \de t} \norm{Y_{t_{k-1}} - \widehat{Z}_{t_{k-1}}}^2\\
        &\quad\quad + \norm{\half \intk e^{\int_t^{t_k} f(T-s) \de s} g^2(T-t) \qty(\nabla \log p_{T-t}(Y_{t_{k-1}}) - \nabla \log p_{T-t}(\widehat{Z}_{t_{k-1}})) \de t}^2\\
        &\quad\quad + 2\,\Bigg\langle e^{\intk f(T-t) \de t} \qty(Y_{t_{k-1}} - \widehat{Z}_{t_{k-1}}),\\
        &\quad\qquad\qquad\quad \half \intk e^{\int_t^{t_k} f(T-s) \de s} g^2(T-t) \qty(\nabla \log p_{T-t}(Y_{t_{k-1}}) - \nabla \log p_{T-t}(\widehat{Z}_{t_{k-1}})) \de t \Bigg\rangle\\
        &\quad\le e^{2 \intk f(T-t) \de t} \norm{Y_{t_{k-1}} - \widehat{Z}_{t_{k-1}}}^2\\
        &\quad\quad + \frac{1}{4} \qty(\intk e^{\int_t^{t_k} f(T-s) \de s} g^2(T-t) \norm{\nabla \log p_{T-t}(Y_{t_{k-1}}) - \nabla \log p_{T-t}(\widehat{Z}_{t_{k-1}})} \de t)^2\\
        &\quad\quad + e^{\intk f(T-t) \de t} \intk e^{\int_t^{t_k} f(T-s) \de s} g^2(T-t)\\
        &\quad\qquad\qquad\qquad\qquad\qquad\qquad \cdot \dotprod{Y_{t_{k-1}} - \widehat{Z}_{t_{k-1}}}{\nabla \log p_{T-t}(Y_{t_{k-1}}) - \nabla \log p_{T-t}(\widehat{Z}_{t_{k-1}})} \de t.
    \end{align*}
    From the weak concavity and Lipschitz continuity of $\nabla \log p_{T-t}$, established in Proposition~\ref{prop:weak_pt} and \ref{prop:weak_Lip}, respectively, it follows that
    \begin{align*}
        \norm{S_1(k,h)}^2 &\le e^{2 \intk f(T-t) \de t} \norm{Y_{t_{k-1}} - \widehat{Z}_{t_{k-1}}}^2\\
        &\quad + \frac{1}{4} \qty(\intk e^{\int_t^{t_k} f(T-s) \de s} g^2(T-t) L(T-t) \norm{Y_{t_{k-1}} - \widehat{Z}_{t_{k-1}}} \de t)^2\\
        &\quad - e^{\intk f(T-t) \de t} \intk e^{\int_t^{t_k} f(T-s) \de s} g^2(T-t)\\
        &\qquad\qquad\qquad\qquad\qquad\qquad \cdot \Big(\alpha(T-t) - M(T-t)\Big) \norm{Y_{t_{k-1}} - \widehat{Z}_{t_{k-1}}}^2 \de t\\
        &= e^{2 \intk f(T-t) \de t} \norm{Y_{t_{k-1}} - \widehat{Z}_{t_{k-1}}}^2\\
        &\quad + \frac{1}{4} \qty(\intk e^{\int_t^{t_k} f(T-s) \de s} g^2(T-t) L(T-t) \de t)^2 \norm{Y_{t_{k-1}} - \widehat{Z}_{t_{k-1}}}^2\\
        &\quad - e^{2 \intk f(T-t) \de t} \intk e^{\int_{t_{k-1}}^{t} f(T-s) \de s} g^2(T-t)\\
        &\qquad\qquad\qquad\qquad\qquad\qquad \cdot \Big(\alpha(T-t) - M(T-t)\Big) \de t \cdot \norm{Y_{t_{k-1}} - \widehat{Z}_{t_{k-1}}}^2.
    \end{align*}
    By Cauchy-Schwartz inequality, it holds that
    \begin{align*}
        \qty(\intk e^{\int_t^{t_k} f(T-s) \de s} g^2(T-t) L(T-t) \de t)^2
        & \le \intk e^{2 \int_t^{t_k} f(T-s) \de s} \de t \cdot \intk g^4(T-t) L^2(T-t) \de t\\
        & \le \intk e^{2 \int_{t_{k-1}}^{t_k} f(T-s) \de s} \de t \cdot \intk g^4(T-t) L^2(T-t) \de t\\
        & \le h e^{2 \int_{t_{k-1}}^{t_k} f(T-s) \de s} \cdot \intk g^4(T-t) L^2(T-t) \de t,
    \end{align*}
    which further yields
    \begin{align*}
        &\norm{S_1(k,h)}^2 \\
        &\qquad\le \qty( 1 - \intk \qty[e^{\int_{t_{k-1}}^{t} f(T-s) \de s} g^2(T-t) \Big(\alpha(T-t) - M(T-t)\Big) - \frac{1}{4} h g^4(T-t) L^2(T-t)] \de t) \\
        &\qquad\qquad \cdot e^{2 \intk f(T-t) \de t} \norm{Y_{t_{k-1}} - \widehat{Z}_{t_{k-1}}}^2 \\
        &\qquad= \qty(1 - \intk 2 \delta_k(T-t) \de t) e^{2 \intk f(T-t) \de t} \norm{Y_{t_{k-1}} - \widehat{Z}_{t_{k-1}}}^2.
    \end{align*}
    Note that, since the left-hand side of this inequality is non-negative, the right-hand side is guaranteed to be non-negative as well. Hence, using the inequality $\sqrt{1-x} \le 1 - \frac{x}{2}$, which holds for any $x \le 1$, we conclude that
    \begin{equation}
    \label{eq:S1_bound}
        \norm{S_1(k,h)}_{L_2} \le \qty(1 - \intk \delta_k(T-t) \de t) e^{\intk f(T-t) \de t} \norm{Y_{t_{k-1}} - \widehat{Z}_{t_{k-1}}}_{L_2}.
    \end{equation}

    By Proposition~\ref{prop:weak_Lip}, we get for the second term that
    \begin{align*}
        \norm{S_2(k,h)}^2 &= \norm{\intk e^{\int_t^{t_k} f(T-s) \de s} g^2(T-t) \qty(\nabla \log p_{T-t}(Y_t) - \nabla \log p_{T-t}(Y_{t_{k-1}})) \de t}^2\\
        &\le \qty(\intk e^{\int_t^{t_k} f(T-s) \de s} g^2(T-t) \norm{\nabla \log p_{T-t}(Y_t) - \nabla \log p_{T-t}(Y_{t_{k-1}})} \de t)^2\\
        &\le \qty(\intk e^{\int_t^{t_k} f(T-s) \de s} g^2(T-t) L(T-t) \norm{Y_t - Y_{t_{k-1}}} \de t)^2.
    \end{align*}
    An application of Cauchy-Schwartz inequality further yields 
    \begin{align*}
        \norm{S_2(k,h)}^2 &\le \intk 1^2 \de t \cdot \intk \qty(e^{\int_t^{t_k} f(T-s) \de s} g^2(T-t) L(T-t) \norm{Y_t - Y_{t_{k-1}}} )^2 \de t\\
        &= h \intk \qty[e^{\int_t^{t_k} f(T-s) \de s} g^2(T-t) L(T-t)]^2 \norm{Y_t - Y_{t_{k-1}}}^2 \de t\\
        &\le h \intk \qty[e^{\int_t^{t_k} f(T-s) \de s} g^2(T-t) L(T-t)]^2 \de t \sup_{t_{k-1} \le t \le t_k} \norm{Y_t - Y_{t_{k-1}}}^2.
    \end{align*}
    It follows that
    \begin{align*}
        \norm{S_2(k,h)}_{L_2} &\le \qty(h \intk \qty[e^{\int_t^{t_k} f(T-s) \de s} g^2(T-t) L(T-t)]^2 \de t \, \mathbb{E}\qty[\sup_{t_{k-1} \le t \le t_k} \norm{Y_t - Y_{t_{k-1}}}^2])^\half\\
        &\le \sqrt{h} \qty(\intk \qty[e^{\int_t^{t_k} f(T-s) \de s} g^2(T-t) L(T-t)]^2 \de t)^\half \sup_{t_{k-1} \le t \le t_k} \norm{Y_t - Y_{t_{k-1}}}_{L_2}\\
        &\le \sqrt{h} \, \nu_{k,h} \, \qty(\intk \qty[e^{\int_t^{t_k} f(T-s) \de s} g^2(T-t) L(T-t)]^2 \de t)^\half,
    \end{align*}
    where for the last inequality, we used Lemma~\ref{lem:nu_kh}.

    For the third term, Assumption \ref{ass:Lip_in_time} implies that 
    \begin{align*}
        \norm{S_3(k,h)}_{L_2} &\le \intk e^{\int_t^{t_k} f(T-s) \de s} g^2(T-t) \norm{\nabla \log p_{T-t}(\widehat{Z}_{t_{k-1}}) - \nabla \log p_{T-t_{k-1}}(\widehat{Z}_{t_{k-1}})}_{L_2} \de t\\
        &\le \intk e^{\int_t^{t_k} f(T-s) \de s} g^2(T-t) L_1 h \qty(1 + \norm{\widehat{Z}_{t_{k-1}}}_{L_2}) \de t\\
        &\le L_1 h \qty(1 + \norm{Y_{t_{k-1}} - \widehat{Z}_{t_{k-1}}}_{L_2} + \norm{Y_{t_{k-1}} - \tilde{X}_{t_{k-1}}}_{L_2} + \norm{\tilde{X}_{t_{k-1}}}_{L_2})\\
        &\qquad \cdot \intk e^{\int_t^{t_k} f(T-s) \de s} g^2(T-t) \de t.
    \end{align*}
    By \eqref{eq:yt-xt_bound}, we have
    \begin{equation*}
        \norm{Y_{t_{k-1}} - \tilde{X}_{t_{k-1}}}_{L_2} \le e^{- \half \int_{t=0}^{t_{k-1}} m(T-s) \de s} e^{-\int_0^T f(s) ds} \norm{\tilde{X}_0}_{L_2} \le \theta(T).
    \end{equation*}
    Moreover, since $\tilde{X}_t = X_{T-t}$ in distribution for any $t \in [0,T]$, Lemma~\ref{lem:omega(T)} implies that
    \begin{equation*}
        \norm{\tilde{X}_{t_{k-1}}}_{L_2} = \norm{X_{T - t_{k-1}}}_{L_2} \le \sup_{t \in [0,T]} \norm{X_t}_{L_2} = \omega(T).
    \end{equation*}
    In summary, this yields
    \begin{align*}
        \norm{S_3(k,h)}_{L_2} &\le L_1 h \qty(1 + \theta(T) + \omega(T) + \norm{Y_{t_{k-1}} - \widehat{Z}_{t_{k-1}}}_{L_2}) \intk e^{\int_t^{t_k} f(T-s) \de s} g^2(T-t) \de t.
    \end{align*}

    The fourth term can be easily bounded by Assumption \ref{ass:score_error}. In particular, we have
    \begin{align*}
        &\norm{S_4(k,h)}_{L_2}\\
        &\qquad \le \intk e^{\int_t^{t_k} f(T-s) \de s} g^2(T-t) \norm{\nabla \log p_{T-t_{k-1}}(\widehat{Z}_{t_{k-1}}) - s_\theta\qty(\widehat{Z}_{t_{k-1}}, T-t_{k-1})}_{L_2} \de t\\
        &\qquad \le \mathcal{E} \intk e^{\int_t^{t_k} f(T-s) \de s} g^2(T-t) \de t.
    \end{align*}

    Combining the bounds for all four summands in \eqref{eq:err_bound_4_summands}, we conclude that
    \begin{align*}
        \norm{Y_{t_k} - \widehat{Z}_{t_k}}_{L_2} &\le \qty(1 - \intk \delta_k(T-t) \de t) e^{\intk f(T-t) \de t} \norm{Y_{t_{k-1}} - \widehat{Z}_{t_{k-1}}}_{L_2}\\
        &\quad + \half \sqrt{h} \, \nu_{k,h} \, \qty(\intk \qty[e^{\int_t^{t_k} f(T-s) \de s} g^2(T-t) L(T-t)]^2 \de t)^\half\\
        &\quad + \half L_1 h \qty(1 + \theta(T) + \omega(T) + \norm{Y_{t_{k-1}} - \widehat{Z}_{t_{k-1}}}_{L_2})\\
        &\qquad \cdot\intk e^{\int_t^{t_k} f(T-s) \de s} g^2(T-t) \de t\\
        &\quad + \half \mathcal{E} \intk e^{\int_t^{t_k} f(T-s) \de s} g^2(T-t) \de t.
    \end{align*}
    Using the fact that
    \begin{equation*}
        \intk e^{\int_t^{t_k} f(T-s) \de s} g^2(T-t) \de t \le e^{\intk f(T-t) \de t} \intk g^2(T-t) \de t 
    \end{equation*}
    and slightly rearranging the terms finally completes the proof. 
\end{proof} 
\bibliographystyle{abbrvnat}
\bibliography{../literature}
\end{document}